%% file: main.tex
\documentclass{article}

\usepackage[preprint]{neurips_2026}

\usepackage[utf8]{inputenc} 
\usepackage[T1]{fontenc}    
\usepackage{booktabs}       
\usepackage{amsfonts}       
\usepackage{nicefrac}       
\usepackage{microtype}      
\usepackage{xcolor}         

\title{Cross-Fitting for Neural Posterior Estimation}

\input{math_commands.tex}

\usepackage{array}
\usepackage{amsthm}
\usepackage{thm-restate}
\usepackage{graphicx}
\usepackage{algorithm}
\usepackage{algorithmic}
\usepackage{float}
\usepackage{hyperref}
\usepackage[capitalize,noabbrev]{cleveref}
\crefformat{section}{\S#2#1#3}
\crefformat{subsection}{\S#2#1#3}

\author{%
Jeffrey Regier\\
Department of Statistics\\
University of Michigan\\
Ann Arbor, MI 48109 \\
\texttt{regier@umich.edu}
}

\begin{document}

\maketitle

\begin{abstract}
Neural Posterior Estimation (NPE) fits an amortized conditional density estimator to a simulated training set of complete-data realizations.
NPE's training objective, log-likelihood, favors posterior approximations that assign arbitrarily high density to the training points, and typical families of amortized conditional density estimators (e.g., normalizing flows, Gaussian encoders) are flexible enough to oblige.
The resulting policy of extreme overconfidence translates poorly to held-out data.
Early stopping and other regularizers can prevent this degeneracy, but they are limited to selecting posterior approximations along an unfavorable optimization trajectory set by in-sample residuals, where the posterior mean is underfit while the posterior variance is overfit.
To avoid this trade-off, we propose Cross-Fitting NPE (CF-NPE), which uses sample splitting to separate conditional-mean estimation from the estimation of the posterior's higher moments.
Using $K$-fold cross-fitting, CF-NPE fits $K$ posterior-mean regressors and then fits a conditional density estimator to the resulting out-of-fold residuals.
In experiments with synthetic data and in a cosmological application, CF-NPE achieves higher held-out log-likelihood than standard NPE and typically improves calibration.
CF-NPE is most useful in high-dimensional, simulation-limited settings.
\end{abstract}

\section{Introduction}

Neural Posterior Estimation (NPE) is a widely used method of simulation-based inference~\citep{papamakarios2016fast,greenberg2019automatic}. 
Given a prior $p(\rvz)$ and a simulator $p(\rvx\mid\rvz)$ from which samples can be drawn, NPE trains an amortized conditional density estimator
$q_\phi(\rvz\mid\rvx)$ with parameters $\phi$ to approximate the posterior.
During training, NPE maximizes the log-likelihood:
\begin{equation*}
\hat{\Ls}(\phi)
= \frac{1}{N}\sum_{i=1}^N \log q_\phi(\rvz_i \mid \rvx_i),
\end{equation*}
where $\{(\rvx_i, \rvz_i)\}_{i=1}^N$ are independent samples from the joint distribution $p(\rvx, \rvz)$.

The amortized conditional density estimator $q_\phi$ is often taken to be 
a Gaussian encoder or a normalizing flow~\citep{papamakarios2021normalizing}.
The approximators typically used, including these two, are sufficiently flexible that, without regularity conditions on the approximating family $\mathcal{Q} = \{q_\phi\}$, the NPE training objective is unbounded above: at each training input $\rvx_i$, the fitted conditional density can place arbitrarily high density near $\rvz_i$, making the empirical log-likelihood arbitrarily large.
Thus, very different predicted variances can achieve the same training log-likelihood, and the posterior approximation selected is determined by the regularity conditions.

The necessity of regularizers to prevent degeneracy is not unique to NPE.
The unbounded-likelihood problem for densities concentrated on observations is well known \citep[Sec.~6]{kiefer1956consistency}.
In regression, without regularity conditions, the training data can be perfectly interpolated by regressors that perform arbitrarily poorly on held-out data.

However, for regression with neural networks, the network architecture and the optimization path of stochastic first-order optimizers act as powerful regularizers; overfitting by neural networks is ``benign'' or ``tempered'' rather than ``catastrophic,'' in the taxonomy of \citet{mallinar2022benign}.
Early stopping is still useful, but it is not what stands in the way of catastrophic validation performance.

Neural conditional density estimation is different.
In NPE, overfitting is better described as catastrophic than as benign or tempered.
With a Gaussian encoder, for example, finite network capacity does not prevent setting standard deviations to the smallest positive value that floating-point numbers can represent.
Early stopping is left to prevent degeneracy, and thus to set the scale of the posterior approximation.

But will early stopping recover the correct posterior scale?
Early stopping effectively performs a one-dimensional search along the optimizer's trajectory, identifying the epoch with the best validation-set performance.
Given that the search space is restricted to this trajectory, this is a reasonable policy: maximizing validation-set log-likelihood minimizes the KL divergence from the posterior to its approximation (that is, the forward KL divergence), averaged over simulated data realizations \citep{ambrogioni2019forward}.

However, the optimizer's trajectory is an unfortunate space to search for a good approximation of the posterior.
This trajectory is produced by fitting the posterior mean and scale simultaneously to the same training pairs.
Because fitting the mean reduces its in-sample residuals, those residuals underestimate held-out residual error for the fitted mean.
As the mean function comes to perfectly interpolate the data points, the residuals, and therefore the fitted variance, approach zero.

\begin{figure*}[!tb]
\centering
\includegraphics[width=\textwidth]{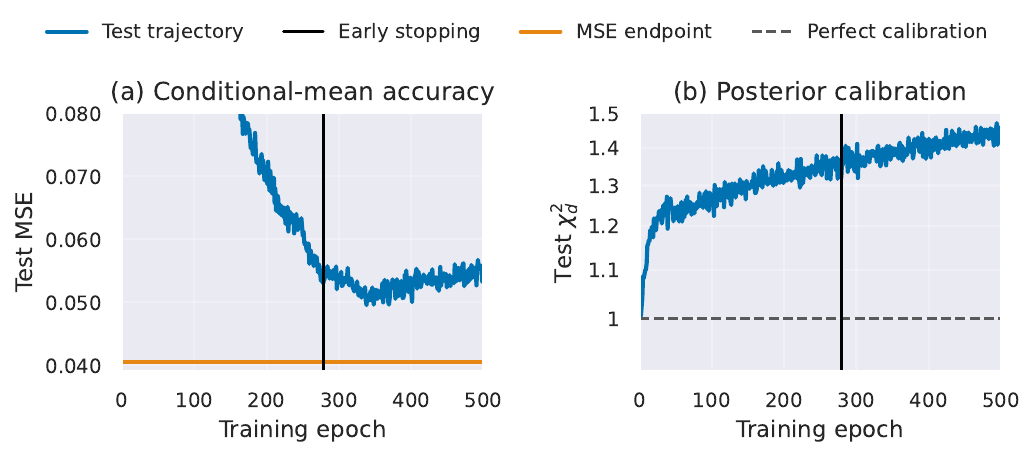}
\caption{Raw per-epoch training trajectories for standard NPE
(three width-256 hidden layers) on the Gaussian--Gaussian problem ($d=d_\rvx=200$,
$N=10{,}000$), shown for one seed. Left:
held-out mean squared error (MSE) to the analytic posterior mean. Right:
$\chi^2_d(q;P)$ over a test set, with the expectation over $\rvz\mid \rvx$
evaluated analytically and averaged over 10,000 fixed test inputs.}
\label{fig:variance_trajectory}
\end{figure*}

\Cref{fig:variance_trajectory} illustrates the shortcomings of this search space for a particular NPE training run.
We quantify posterior calibration with the reduced chi-square statistic:
\begin{equation}
\label{eq:reduced_chi2}
\chi^2_d(q;P)
:=\E_P\!\left[
\frac{1}{d}\sum_{\ell=1}^{d}
\frac{(\ervz_\ell-\vmu_{q,\ell}(\rvx))^2}{\sigma_{q,\ell}^2(\rvx)}
\right],
\end{equation}
where $d$ is the dimension of the latent space, $q$ is the fitted posterior approximation with conditional mean $\vmu_q(\rvx):=\E_q[\rvz\mid \rvx]$ and coordinate-wise variances $\sigma_{q,\ell}^2(\rvx):=\Var_q(\ervz_\ell\mid \rvx)$, and $P$ is a generic distribution over pairs $(\rvx,\rvz)$.
The reduced chi-square equals one when the fitted variances match the squared errors about the fitted mean under $P$.
Here we take $P$ to be the empirical distribution of a held-out test set.

In \Cref{fig:variance_trajectory}, posterior mean estimation improves for 337 epochs before overfitting dominates (left panel), whereas posterior calibration steadily worsens across epochs (right panel), as the posterior spread is trained in-sample with increasingly overfit means.
Early stopping (based on validation log-likelihood) trades off between mean estimation and calibration, selecting epoch 278, which is optimal for both validation- and test-set log-likelihood given the search space.
However, the MSE endpoint line (left panel) shows that substantially better posterior mean estimation is possible.
The MSE endpoint is the held-out MSE of a network with a similar architecture trained to minimize MSE.
It is lower than the NPE trajectory's MSE not only at epoch 278 but also at its minimum, epoch 337.
The MSE endpoint benefits from benign overfitting, with early stopping responding solely to the quality of conditional mean estimation. 

To avoid the trade-off NPE faces between the accurate estimation of the posterior's first and second moments, we propose Cross-Fitting NPE (CF-NPE).
CF-NPE fits the conditional mean and residual density in separate stages using $K$-fold cross-fitting (\Cref{sec:method}).
Stage~1 partitions the data into $K$ folds and trains $K$ posterior-mean regressors, each with a different fold held out. 
These regressors produce out-of-fold residuals, and stage~2 fits the conditional distribution of these residuals.

On three high-dimensional synthetic inference problems,
CF-NPE achieves higher test log-likelihood than standard NPE
(\Cref{sec:experiments}). Its posterior calibration, assessed by $\chi^2_d$, is within a few percent of one in every case except one, where it is 1.15, whereas standard NPE's is as high as 1.51.
On a strong-gravitational-lensing source-reconstruction problem, CF-NPE
achieves higher test log-likelihood than standard NPE and compute-matched standard NPE mixtures (\Cref{sec:strong_lensing}).

CF-NPE is most beneficial when the simulator is computationally expensive and the latent space dimension is high.
If the simulator is cheap enough to draw fresh pairs at every optimizer step, the deficiency of standard NPE motivating this work is not exercised (\Cref{sec:discussion}).

\section{Background}
\label{sec:background}

\paragraph{Sample splitting and cross-fitting.}
The jackknife \citep{quenouille1949approximate,efron1982jackknife} recomputes a statistic on samples that omit
each observation in turn.
$K$-fold cross-fitting is a related procedure that omits one fold at a time: a first-stage model is
trained on $K-1$ folds and used to predict on the held-out fold, producing out-of-fold predictions for every observation that serve as inputs to a second-stage estimator.
Each out-of-fold prediction is computed without the target observation, removing the same-pair dependence between stages. In econometrics, double machine learning~\citep{chernozhukov2018double} combines
cross-fitting with estimating equations for observational causal inference.

\paragraph{Deep ensembles.}
Like us, \citet{hermans2022crisis} observe that NPE and other simulation-based inference methods can produce overconfident posterior approximations. However, \citet{hermans2022crisis} do not trace this problem to in-sample residuals. Instead, they propose that the deep ensembles method of \citet{lakshminarayanan2017simple}, developed for probabilistic prediction more generally, can correct NPE overconfidence.
In deep ensembles, the final posterior approximation is a uniform mixture of component posterior approximations, each fitted independently using the NPE objective with different random weight initializations.
The final posterior approximation is typically more dispersed than its components, so deep ensembles do indeed counteract some underdispersion, but not in a calibrated way.
For example, mixing posterior approximations can improve posterior mean estimates without contracting posterior variances accordingly, resulting in underconfidence.

\paragraph{Heteroscedastic regression.}
When paired with a Gaussian encoder, NPE essentially performs heteroscedastic regression \citep{nix1994estimating} on the simulated training pairs.
Recent methods for heteroscedastic regression with neural networks modify the joint loss or its gradient updates to avoid poor mean or variance estimates through alternating updates \citep{detlefsen2019reliable}, loss reweighting \citep[$\beta$-NLL;][]{seitzer2022pitfalls},
or stop-gradient operations~\citep{stirn2023faithful}.
These methods alter the training loss but still estimate means and scales from the same training pairs.

\section{Method}
\label{sec:method}
To fit an amortized posterior approximation, CF-NPE first partitions the simulated training data $\train = \{(\rvx_i, \rvz_i)\}_{i=1}^N$ uniformly at random into $K$ equal-sized folds.
For each training pair $i = 1, \ldots, N$, let $k(i)$ denote the fold containing training pair $i$.
For each fold $k = 1, \ldots, K$, let $\train_{k} := \{(\rvx_i, \rvz_i) : k(i) = k\}$ and let $\train_{-k} := \train \setminus \train_k$.
CF-NPE then fits an amortized posterior approximation in two stages.

\paragraph{Stage 1 training.}
For each fold $k = 1, \ldots, K$, stage~1 trains a neural network regressor $\hat\vmu_{-k}(\cdot)$ to predict the posterior mean of $\rvz \in \R^{d}$ for any $\rvx \in \R^{d_\rvx}$.
The training objective uses the data subset $\train_{-k}$ and MSE loss, which is asymptotically minimized by the conditional expectation.
Each regressor $\hat\vmu_{-k}$ is either a single network ($M=1$) or the average of an ensemble of $M>1$ networks, all trained on $\train_{-k}$ but with different random initializations of their weights and biases.
The $KM$ stage-1 networks may be trained sequentially or in parallel.

\paragraph{Stage 2 training.}
For each training pair $i = 1, \ldots, N$, let the out-of-fold residual $\rvr_i = \rvz_i - \hat\vmu_{-k(i)}(\rvx_i)$.
Stage~2 fits a conditional density estimator using these out-of-fold residuals.
The residual density $q_\psi$ is parameterized by a neural network with weights $\psi \subset \phi$.
It is trained to maximize
\begin{equation}
\label{eq:stage2_objective}
\Ls := \frac{1}{N} \sum_{i = 1}^N \log q_\psi\left(\rvr_i \mid \rvx_i, \hat\vmu_{-k(i)}(\rvx_i)\right).
\end{equation}
The residual density $q_\psi$ can be a mean-zero diagonal Gaussian, a normalizing flow, or any other conditional density family.

Alternatively, in a variant of CF-NPE that we call \emph{CF-NPE-mixture}, we repeat this stage-2 training procedure $K$ times to produce $\psi_1, \ldots, \psi_K$, each with a different weight initialization and a different validation fold.

\paragraph{Inference.}
At inference time, CF-NPE pairs the stage-2 residual density estimator with the mean regressor of an arbitrary fold, which we take to be fold~1 for concreteness:
\begin{equation}
\label{eq:cf_npe_base}
q_\phi(\rvz \mid \rvx) = q_\psi\left(\rvz - \hat{\vmu}_{-1}(\rvx) \mid \rvx,\, \hat{\vmu}_{-1}(\rvx) \right).
\end{equation}
Here $\phi$ collects the stage-1 and stage-2 parameters.
In CF-NPE-mixture, a uniform mixture is used instead:
\begin{equation}
\label{eq:cf_npe_offset_ensemble}
q_\phi(\rvz \mid \rvx) = \frac{1}{K}\sum_{k=1}^{K}
  q_{\psi_k}\left(\rvz - \hat{\vmu}_{-k}(\rvx) \mid \rvx,\, \hat{\vmu}_{-k}(\rvx)\right).
\end{equation}

\begin{algorithm}[t]
\caption{CF-NPE and CF-NPE-mixture.}
\label{alg:cf_npe}
\begin{algorithmic}[1]
\REQUIRE Training pairs $\train=\{(\rvx_i,\rvz_i)\}_{i=1}^N$; number of folds $K$;
  stage-1 replicates $M$
\STATE Partition $\train$ uniformly at random into equal folds
  $\train_1,\ldots,\train_K$
\FOR{$k=1,\ldots,K$}
  \STATE Fit $M$ mean networks on $\train\setminus\train_k$
    and average them to obtain $\hat\vmu_{-k}$
\ENDFOR
\STATE Form the triples $\mathcal T=\{(\rvx_i,\hat\vmu_{-k(i)}(\rvx_i),
  \rvz_i-\hat\vmu_{-k(i)}(\rvx_i))\}_{i=1}^N$, with $k(i)$ the fold of pair $i$
\STATE Fit the base residual density $q_\psi(\rvr\mid \rvx,\vm)$ to $\mathcal T$ by maximizing \Cref{eq:stage2_objective}
\STATE \emph{CF-NPE-mixture:} Refit $\psi$ an additional $K-1$ times with different initializations and held-out residual folds
\RETURN Base posterior (\Cref{eq:cf_npe_base}) or mixture (\Cref{eq:cf_npe_offset_ensemble})
\end{algorithmic}
\end{algorithm}

\Cref{alg:cf_npe} summarizes the proposed method.

\subsection{The CF-NPE advantage}
\label{sec:theory}

Intuitively, by decoupling posterior mean estimation (stage~1) from estimation of the posterior's higher-order moments (stage~2), CF-NPE avoids the pitfall of standard NPE illustrated in \Cref{fig:variance_trajectory}. By minimizing the MSE, which recovers the conditional mean in the population limit, stage~1 attains the ``MSE endpoint'' line and uses the corresponding means to center the posterior approximation. The early-stopping epoch used for stage~1 (mean estimation) is not shared by stage~2.
Stage~2 then fits a conditional density estimator using out-of-fold residuals. Because these residuals are produced through cross-fitting, their distribution matches that of residuals of held-out responses at the same inputs.

Some of this intuition can be formalized. Our first result examines the difference between out-of-sample and in-sample prediction error.
\begin{restatable}[Optimism of In-Sample Prediction]{proposition}{cfnpecalibration}
\label{thm:cross_fitting}
Let $(\rvx_i,\rvz_i)_{i=1}^N\stackrel{\mathrm{iid}}{\sim}p(\rvx,\rvz)$ be training pairs, with $\E[\|\rvz\|^2]<\infty$.
For $i=1,\ldots,N$, let $\vc_i\in\R^d$ be a predicted posterior mean with finite second moment, computed by an arbitrary procedure from the training pairs and exogenous training randomness (e.g., weight initialization, fold assignment, dropout sequence, and minibatching), and let $\rvz'_i$ be a fresh draw from $p(\rvz\mid \rvx_i)$, independent of $\rvz_{1:N}$ and of the training randomness.
Then for each coordinate $\ell=1,\ldots,d$ and almost every $\rvx_{1:N}$,
\begin{equation}
\label{eq:standard_npe_cross_term}
\E[(\ervz'_{i\ell}-c_{i\ell})^2\mid \rvx_{1:N}] - \E[(\ervz_{i\ell}-c_{i\ell})^2\mid \rvx_{1:N}]
= 2\Cov\big(\ervz_{i\ell},c_{i\ell}\mid \rvx_{1:N}\big).
\end{equation}
\end{restatable}
The proof closely follows a derivation of the optimism of in-sample prediction error rates from \citet{efron2004estimation}. All proofs appear in \Cref{app:oof}.

In standard NPE, each posterior approximation's mean $\vc_i$ is fitted using all training points, so the covariance term in \Cref{eq:standard_npe_cross_term} is non-zero and tends to grow as optimization progresses due to overfitting.
In CF-NPE, in contrast, the posterior approximation mean $\vc_i=\hat\vmu_{-k(i)}(\rvx_i)$ is fitted using the $K-1$ folds that do not contain $\rvz_i$.
Thus, the covariance term is zero, and the in-sample and out-of-sample expected squared residuals agree.
Their shared value is lower-bounded by the posterior variance, as the following proposition shows.
\begin{restatable}[Residual Second Moment]{proposition}{cfnperesidualmoment}
\label{thm:residual_moment}
For $(\rvx,\rvz)\sim p(\rvx,\rvz)$ with $\E[\|\rvz\|^2]<\infty$, write
$\vmu^\star(\rvx):=\E[\rvz\mid \rvx]$ and $\mSigma^\star(\rvx):=\Cov(\rvz\mid \rvx)$.
Then for any deterministic vector $\vc\in\R^d$, any coordinate $\ell=1,\ldots,d$, and almost every $\rvx$,
\begin{equation}
\label{eq:residual_moment_target}
\E[(\ervz_\ell-c_\ell)^2\mid \rvx]
=\mSigma^\star_{\ell\ell}(\rvx)+(c_\ell-\vmu^\star_\ell(\rvx))^2.
\end{equation}
\end{restatable}
Thus, the cross-fitted squared residual estimates the sum of the posterior variance and the squared error of the stage-1 regressor.
As the stage-1 regressor converges to $\vmu^\star$, the second moment that stage~2 targets converges to the posterior variance, for almost every $\rvx$.

For CF-NPE to learn a calibrated posterior approximation, however, it is not enough to avoid unwarranted optimism in the residuals used in stage~2. 
Because CF-NPE stage~2 is itself standard NPE applied to the residual dataset, a natural concern is that it inherits the limitations of standard NPE that motivated this work. 
Thankfully, the residuals, unlike the original training responses, are approximately centered at the origin---a property a Gaussian encoder can exploit by construction and normalizing flows through their initializations.

\paragraph{Calibration with centered Gaussians.}

Suppose that CF-NPE's family of posterior approximations is Gaussian with a diagonal covariance, and further that CF-NPE stage~2 fits a mean-zero Gaussian distribution to the residuals: $q_\psi(\rvr\mid \rvx,\vm)=\mathcal N(\rvr;\vzero,\operatorname{diag}(V_\psi(\rvx,\vm)))$. Then, the posterior approximation that maximizes CF-NPE's stage-2 objective is perfectly calibrated on the training set, as the next result shows.
\begin{restatable}[Reduced Chi-Square at a Likelihood Maximum]{proposition}{cfnpescale}
\label{prop:scale_calibration}
Let $P_{\mathrm{train}}$ be the empirical distribution $N^{-1}\sum_{i=1}^N\delta_{(\rvx_i,\rvz_i)}$, and let
$\mathcal Q$ be a family of diagonal Gaussian conditional densities, closed under common variance rescaling with the mean fixed.
If $q\in\mathcal Q$ maximizes $\E_{P_{\mathrm{train}}}[\log q(\rvz\mid \rvx)]$ with finite objective value, then $\chi^2_d(q;P_{\mathrm{train}})=1$.
\end{restatable}
Hence, for CF-NPE, but not standard NPE, calibration on the training set implies calibration on test data by \Cref{thm:cross_fitting}, which equates the expected in-sample and out-of-sample squared residuals.

\paragraph{Calibration with normalizing flows.}
Typical normalizing flows do not provide a way of centering the posterior approximation by construction. Nevertheless, there is reason to expect that they too will benefit from training with centered responses. With a centered base distribution and symmetric weight and bias initialization, both of which are the norm, a normalizing flow produces approximately mean-zero posterior approximations initially. Because the mean starts near its optimum and stays near it along the optimization path, early stopping can stop optimization as soon as the higher moments are estimated well without compromising mean-estimation quality.

\section{Experiments}
\label{sec:experiments}

We benchmark several variants of CF-NPE and standard NPE with two families of amortized conditional density estimators: diagonal Gaussian encoders and neural spline flows (NSF) \citep{durkan2019neural}, implemented with Zuko \citep{francois_rozet_2024_10571785}.%
\footnote{
Software to reproduce all experiments is available from \url{https://github.com/jeff-regier/cf-npe}.
}

\paragraph{Problems.}
We consider three classes of inference problems.
The first, the Gaussian--Gaussian problem, is a hierarchical model with a 200-dimensional latent space and a 200-dimensional data space. 
The second, the funnel problem, is the pathological funnel from \citet{neal2003slice}, with 200-dimensional latent and data spaces.
The third, the mixture problem, involves a mixture of two Gaussian components, with a 20-dimensional latent space and a 40-dimensional data space.
All three problems' simulators have hyperparameters, such as latent and data dimensions, which we set to make the problems interesting: hard enough that some inference methods would struggle, but not so hard that the problem is hopeless for every method. Unless otherwise specified, we use $N=10{,}000$ training pairs and $1{,}000$ test pairs. For each of 10 seeds, the same simulations are used across methods.
\Cref{app:synthetic_problems} gives further details of each inference problem.

\paragraph{Metrics.}
Our primary metric is the held-out log-likelihood, which, as a proper scoring rule \citep{gneiting2007strictly}, jointly assesses calibration and precision. Log-likelihood is a uniquely important metric in our benchmarking because standard NPE is our chief comparator, and in standard NPE, maximizing log-likelihood, which equivalently minimizes the expected KL divergence from the posterior to its approximation \citep{ambrogioni2019forward}, is what connects the posterior approximation to the posterior.
It is only because standard NPE minimizes this KL divergence that we expect standard NPE posterior approximations to be at least somewhat calibrated and to perform well on secondary metrics of posterior approximation quality.

Our secondary metrics include the MSE of the posterior mean (relative to the simulation truth) and $\chi^2_d$, both of which were considered in our earlier theoretical analysis. Additional calibration metrics reported include SBC-KS \citep{talts2018validating}, TARP \citep{lemos2023sampling}, and C2ST \citep{lopezpaz2017revisiting}.

\paragraph{Hyperparameter selection.}
Hyperparameter tuning is critical for these benchmarks because they seek to measure network fragility due to overfitting.
Thus, the results depend heavily on how the networks are tuned.
We use the Optuna hyperparameter optimizer \citep{akiba2019optuna} to tune all networks, searching over network depth, width, learning rate, weight-decay strength, dropout rate, and batch size.
For NSF, we also search over the number of coupling transforms, spline bins, and a tail-bound parameter.
We run every network with the best-performing configuration identified across 100 trials (20 trials of random search followed by refinement).
Each network is re-tuned for each problem.
Mixture and ensemble variants are not tuned separately but inherit the base method's hyperparameters.
Unless otherwise specified, CF-NPE performs cross-fitting with $K=10$ folds.

\begin{table}[t]
\centering
\scriptsize
\setlength{\tabcolsep}{2.5pt}
\renewcommand{\arraystretch}{0.88}
\caption{Gaussian--Gaussian problem ($d=200$). Held-out log-likelihood (nats, mean
$\pm$ standard error over 10 seeds), $\chi^2_d$ (target 1), and
posterior-mean MSE under each approximating family. MSE standard errors are at most
$0.001$, except $0.003$ for Std-NPE (standard NPE) with the NSF approximation. Bold marks entries
within one standard error of the best within an approximating family. The reference rows are analytic.}
\label{tab:synthetic_gauss_d200}
\input{tables/tab_synthetic_gauss_d200}

\vspace{14pt}
\caption{Funnel problem ($d=200$). Columns as in
\Cref{tab:synthetic_gauss_d200}. MSE standard errors are at most
$0.006$, except $0.015$ for Std-NPE with the NSF approximation. The posterior row exploits the funnel's known form, which the networks are not given: conditional on $z_1$, the other coordinates have Gaussian posteriors, leaving a one-dimensional integral over $z_1$ (\Cref{app:synthetic_problems}).
Means and standard errors are averaged over 9 seeds; one seed is excluded for each method because its test set contains a pair with $z_1 = 4.54$, far beyond the training-data range, which dominates $\chi^2_d$.}
\label{tab:synthetic_funnel_d200}
\input{tables/tab_synthetic_funnel_d200}

\vspace{14pt}
\caption{Mixture problem ($d=20$). Columns as in \Cref{tab:synthetic_gauss_d200}. MSE standard errors are at most $0.001$.}
\label{tab:synthetic_mixture_d20}
\input{tables/tab_synthetic_mixture_d20}
\end{table}

\paragraph{Main result.}
CF-NPE ($M=1$) achieves higher test log-likelihood than standard NPE across all problems and both approximating families (\Crefrange{tab:synthetic_gauss_d200}{tab:synthetic_mixture_d20}).
The increase is consistent across seeds.
While nats are difficult to interpret in absolute terms, varying the training-set size indicates the practical significance of this improvement: CF-NPE ($M=3$) achieves higher test log-likelihood with an $N=10{,}000$ simulation budget than standard NPE does with $N=30{,}000$ simulations (\Cref{app:gg_ntrain}).

The higher test log-likelihood reflects improved recovery of both the first and second moments of the posterior distribution, as \Cref{sec:theory} anticipates.
MSE is lower for CF-NPE in every cell. For CF-NPE, $\chi^2_d$ is within a few percent of one in every cell except the funnel problem with the NSF family, where it is at most 1.15; for standard NPE it is as high as 1.51.
SBC-KS, TARP, and C2ST also find better calibration for CF-NPE in most cells (\Cref{app:synthetic_calibration}).

CF-NPE uses more compute, but the difference scarcely matters in absolute terms in these experiments, as each training run takes under 10 minutes (\Cref{app:synthetic_runtime}). These experiments, unlike the application in \Cref{sec:strong_lensing}, do not saturate the GPU, so CF-NPE trains all folds in parallel (using \texttt{vmap} in PyTorch).

\paragraph{Mixtures and ensembles.}
Deep ensembles (see \Cref{sec:background}) give standard NPE a way to better leverage parallel compute resources and trade compute for fidelity by forming a mixture of standard NPE posterior approximations trained with different weight-initialization seeds.
These mixtures appear as \emph{Std-NPE-mixture} in \Crefrange{tab:synthetic_gauss_d200}{tab:synthetic_mixture_d20}, with $K'$ denoting the number of ensemble members (mixture components).
We set $K'$ to provide approximately compute-matched counterparts to the CF-NPE variants.
Standard NPE benefits substantially from mixtures in terms of test log-likelihood.

However, CF-NPE can also leverage ensembling.
With $M > 1$, the stage-1 (mean) networks are ensembled, while stage~2, and thus the overall posterior approximation, remains a single Gaussian or NSF.
The CF-NPE-mixture variant additionally ensembles $K=10$ stage-2 distributions, forming a uniform mixture over them. 
CF-NPE performance improves as $M$ increases, suggesting stage-1 fitting is sensitive to initialization (capacity is already tuned). Forming a stage-2 mixture of NSF approximators helps with the funnel and mixture problems and is otherwise neutral.
With $M = 3$, CF-NPE outperforms the Std-NPE-mixture variants, which are roughly compute-matched.

The mixture variants of standard NPE are overdispersed in every cell, with $\chi^2_d$ as low as 0.20.
Forming a mixture improves mean estimation without reducing spread, creating overdispersion (as anticipated in \Cref{sec:background}).
In contrast, setting $M > 1$ with CF-NPE improves mean estimation by making the fit more robust to seed noise, without disturbing calibration.
Because runtime is dominated by stage-2 fitting with the NSF posterior, and stage~1 is parallelized over both $K$ folds and $M$ replicates, increasing $M$ from 1 to 3 increases runtime by at most 5 seconds (\Cref{app:synthetic_runtime}).

\paragraph{Ablations.}
On the Gaussian--Gaussian problem, we ablate the ingredients of CF-NPE, vary the training-set size, and test alternative remedies for standard NPE.
Cross-fitting is an essential ingredient: fitting stage~2 to in-fold rather than out-of-fold residuals leaves the posterior mean unchanged but makes the posterior overconfident ($\chi^2_d \approx 1.15$), as \Cref{sec:theory} predicts. A single stage-1 fit without folds or replicates is also overconfident and estimates the mean less accurately (\Cref{app:gg_ablation}).
Log-likelihood rises monotonically with the number of folds $K$ from 2 to 16, and even $K=2$ outperforms standard NPE (\Cref{app:gg_kfold}).
A substantially oversized stage-2 NSF lowers test log-likelihood but leaves $\chi^2_d$ within 0.01 of one (\Cref{app:flow_size}).
CF-NPE outperforms standard NPE at every training-set size from $N=1{,}000$ to $N=100{,}000$ (\Cref{app:gg_ntrain}).
CF-NPE's advantage also depends on the latent dimension: the gap is negligible at $d=2$ and largest at the dimension of each problem's main experiment, though it does not grow monotonically with $d$ (\Cref{app:synthetic_dimsweep}).
None of the alternative remedies for standard NPE closes the gap to CF-NPE. Temperature scaling corrects $\chi^2_d$ but recovers little log-likelihood (\Cref{app:gg_temperature}); $\beta$-NLL and stop-gradient losses change little (\Cref{app:gg_gradient}); and training each Std-NPE-mixture component on a different fold split, as CF-NPE's stage-1 networks are, raises log-likelihood, but by less than half of the mixture's gap to CF-NPE, and leaves the mixture overdispersed (\Cref{app:gg_split_by_fold}).

\section{Application: Strong-Lensing Source Reconstruction}
\label{sec:strong_lensing}

Large masses curve spacetime. Thus, distant galaxies appear sheared or distorted when massive halos known as ``strong lenses'' lie along the line of sight \citep{treu2010strong} (see \Cref{fig:strong_lensing_recon}). Strong lenses are important cosmological probes for three reasons: (1) they act as magnifying lenses, revealing structure in the source plane that otherwise would be imperceptible; (2) they can show multiple images of the same background galaxy, with time delays that measure the expansion rate of the universe; and (3) the extent of the lensing can reveal the distribution of dark matter along various lines of sight.
Inferring the unlensed source plane is important for these applications, but recovery is inherently probabilistic due to pixelization, shot noise, atmospheric point spread, and degeneracy (non-identifiability) between the lens and the unlensed source. 
Upcoming surveys will deliver on the order of $10^5$ lenses \citep{collett2015population}, prompting interest in amortized inference methods \citep{hezaveh2017fast}.

\begin{figure}
\centering
\includegraphics[width=\linewidth]{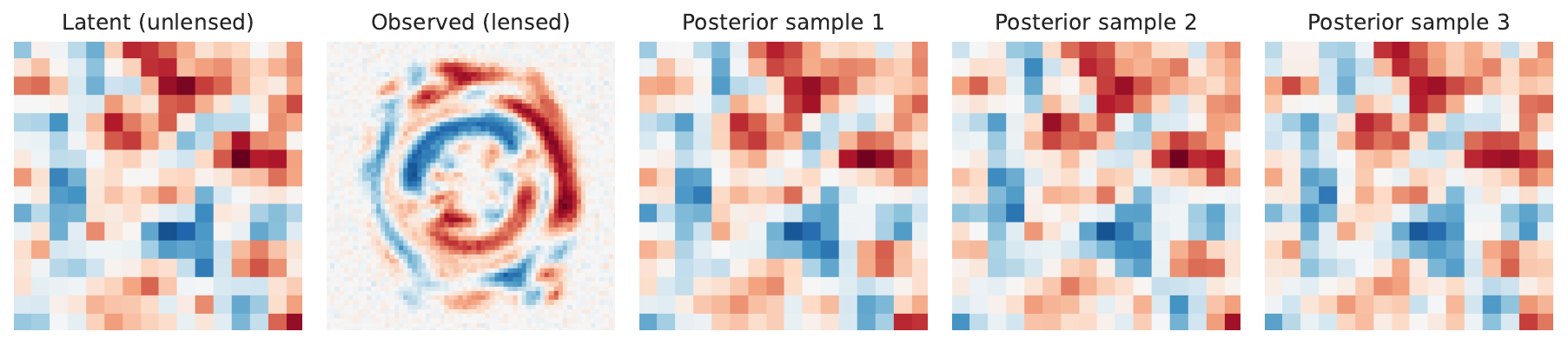}
\caption{Strong-lensing source reconstruction. From left: the true unlensed
source field of one held-out example, its observed lensed image, and three
posterior samples from CF-NPE-mixture with $M=3$.}
\label{fig:strong_lensing_recon}
\end{figure}

We use a five-parameter lens model, a singular isothermal ellipsoid (SIE) \citep{kormann1994isothermal} with external shear, and use Lenstronomy \citep{birrer2018lenstronomy} to render $64\times64$-pixel observed images with Gaussian readout noise, given a source plane and lens.
The latent source plane is specified as a $16\times16$ grid of source pixels, which are 
drawn from a Gaussian process.
The nuisance lensing parameters are sampled uniformly for each example and marginalized over during inference.
We train on $10{,}000$ pairs per seed from this data-generating process.
Lenstronomy is fast, which keeps our experiments inexpensive, but it serves here as a stand-in for high-fidelity simulators such as hydrodynamical simulations of lens galaxies, which can cost over a million core-hours each \mbox{\citep{richings2021high}}; with such simulators, the simulation budget, not training, limits the size of the training set.
Simulation products are cached per seed and shared across methods; reported runtimes exclude simulation.

We apply both standard NPE and CF-NPE to these data using an NSF family of conditional density estimators.
The NSF is composed of coupling layers that alternate between the two halves of a column-striped partition of the source pixels.
Our encoder architecture, shared by all methods, is a residual convolutional neural network with two stride-2 stages that maps the $64 \times 64$-pixel observed image to a $16 \times 16$-pixel representation of the source plane. The network architecture has hyperparameters that enable attention layers, spatial dropout, stochastic-depth regularization, and squeeze-and-excitation, and allow variation in the depth and width of various modules.
We select hyperparameters for each encoder using Optuna and list them in \Cref{app:exp:strong_lensing}.

\begin{table}[!t]
\centering
\scriptsize
\setlength{\tabcolsep}{4pt}
\caption{Strong-lensing source reconstruction with an NSF posterior, 10 seeds per
method. Test log-likelihood is the median $[$interquartile range$]$ over seeds.
Posterior-mean MSE and $\chi^2_d$ are averages over seeds $\pm$ standard error.
Runtime is the median training wall-clock time per seed, excluding simulation time, measured on an NVIDIA GeForce RTX 5090 GPU. CF-NPE rows show stage~1 $+$ stage~2 runtime.}
\label{tab:strong_lensing}
\input{tables/tab_strong_lensing}
\end{table}

\paragraph{Results.}
CF-NPE delivers major improvements over standard NPE (\Cref{tab:strong_lensing}).
Every CF-NPE variant has a higher median test log-likelihood, our primary metric, than standard NPE, with non-overlapping interquartile ranges. Training CF-NPE with $M=1$ takes roughly an hour, whereas standard NPE takes less than 10 minutes, because CF-NPE requires training $K=10$ stage-1 networks.
Unlike in \Cref{sec:experiments}, these CF-NPE networks individually saturate the GPU, so they cannot be sped up through parallelization.

The mixture variants of standard NPE provide compute-matched comparators to CF-NPE.
These variants improve on the non-mixture version, providing higher test log-likelihoods.
However, CF-NPE can also leverage ensembling and mixtures, and the gains from doing so stack with the gains from cross-fitting. With an ensemble of $M=3$ stage-1 networks per fold across $K=10$ folds, CF-NPE delivers a higher test log-likelihood than any standard NPE mixture tested, including the $K'=50$ variant, which trains more networks than CF-NPE with $M=3$ and takes longer to train. 

CF-NPE-mixture, which forms a mixture of stage-2 posterior approximations in addition to ensembling over stage-1 networks, achieves the highest test log-likelihood of all. Training the mixture adds little runtime, as most of the runtime is devoted to the stage-1 ensemble.

In standard NPE, early stopping trades off underfitting the posterior mean against overfitting the posterior variance; CF-NPE removes this trade-off. On this problem, the benefit appears as more accurate posterior means and better calibration: every CF-NPE variant has $\chi^2_d$ within 0.03 of one, whereas standard NPE is overconfident and the Std-NPE-mixture variants are underconfident.

Appendix \Cref{app:exp:strong_lensing} presents further calibration metrics. SBC-KS agrees with $\chi^2_d$: every CF-NPE variant scores better than every standard NPE variant.
Standard NPE attains the lowest TARP KS, but TARP is difficult to interpret when, as here, the posterior is far more concentrated than the prior. TARP draws a reference point from the prior and compares its distances to the simulated latent and to the posterior draws. Two common errors can offset in this comparison: if the approximation's mean is farther from the prior mean than the exact posterior mean is, the draws tend to lie farther from the reference point than the simulated latent; if the approximation is too narrow, the simulated latent tends to lie farther away than the draws.
C2ST is near $0.5$ for every method and is therefore uninformative.

\section{Discussion}
\label{sec:discussion}

CF-NPE corrects a problem with standard NPE that occurs when a simulated training set is reused for multiple epochs, causing the mean estimates to overfit and the in-sample residuals to become underdispersed proxies for the out-of-sample residuals.
This kind of multi-epoch training is standard practice in NPE, but if simulation is cheap enough to draw fresh training pairs at each optimizer step, then in-sample and out-of-sample residuals have the same distribution, and standard NPE's defect is not exercised.
Similarly, if the latent-space dimension is low relative to the number of training examples, standard NPE may not overfit the posterior means and therefore perform like CF-NPE.

\paragraph{Compute trade-off.}
CF-NPE generally requires more training time than standard NPE because it trains multiple networks.
For many problems, training time is small relative to simulation time, or small in absolute terms, and thus of little consequence.
Occasionally, CF-NPE can even train faster than standard NPE because the stage-1 networks are lightweight and can be trained in parallel, and decoupling mean estimation from conditional density estimation favors a lightweight stage-2 network.
More typically, however, when training each network saturates the GPU, CF-NPE is approximately an order of magnitude slower than standard NPE because it trains $O(K)$ networks.
For hyperparameter selection, which can be much more expensive than training, CF-NPE can tune a single ensemble member rather than the full ensemble collectively, giving roughly the same tuning cost as standard NPE.

\paragraph{Sequential NPE.}
Our present analysis covers only single-round amortized NPE.
In sequential NPE methods such as SNPE or APT~\citep{papamakarios2016fast, greenberg2019automatic}, each round's proposal is fixed given the previous rounds, so simulations within a round are independent.
CF-NPE could replace standard NPE within each round, perhaps yielding good empirical results.
However, we have not explored this possibility, as the dependence across rounds makes the entire procedure difficult to analyze, whether CF-NPE performs each round or not.

\paragraph{Cross-fold ensembling.}
CF-NPE averages $M$ networks to estimate a posterior mean, rather than all $KM$ stage-1 networks. At training time, restricting averaging to $M$ networks is needed to implement cross-fitting, but at inference time, the stage-2 network has already been fit, so in principle we could average across all $KM$ networks to get more stable mean estimates.
The drawback of doing so is that improving the mean without shrinking the variance causes underconfidence, which is acceptable in some applications but not others.

\paragraph{Calibration with normalizing flows.}
The theoretical case for CF-NPE's calibration is strongest with a Gaussian encoder, whereas the analysis for normalizing flows leans more on intuition (\Cref{sec:theory}).
Nevertheless, CF-NPE with normalizing flows appears to be well calibrated empirically; even in an ablation study in which the normalizing flows are poorly tuned (oversized and under-regularized), CF-NPE remains nearly perfectly calibrated (\Cref{app:flow_size}).
We would like to better understand why.
The complicated training dynamics of discrete normalizing flows, as considered here, make them difficult to analyze, but continuous normalizing flows trained with flow matching \citep{wildberger2023flow} may be more amenable to analysis.

\paragraph{Higher standardized moments.}
CF-NPE separates mean (first moment) estimation from estimation of the higher moments, avoiding a trade-off between first and second moments. All the higher moments are estimated together, however, by standard NPE in stage~2, presumably entailing a tension between the second and third moments like the one that motivated this work.
Nested forms of cross-fitting could potentially extend the benefits of CF-NPE to these higher moments, but at rapidly escalating computational cost. 
Exploring whether approximate refitting methods \citep[e.g.,][]{giordano2019swiss} can recover much of that benefit at a fraction of the cost is another interesting direction for future work.

\raggedbottom
\bibliographystyle{abbrvnat}
\bibliography{references}

\newpage
\appendix
\begin{center}
{\Large\bfseries Appendices}
\end{center}
\vspace{1em}
\renewcommand{\thefigure}{\Alph{section}.\arabic{figure}}
\renewcommand{\theHfigure}{\Alph{section}.\arabic{figure}}
\setcounter{figure}{0}
\renewcommand{\thetable}{\Alph{section}.\arabic{table}}
\renewcommand{\theHtable}{\Alph{section}.\arabic{table}}
\setcounter{table}{0}
\renewcommand{\theproposition}{\Alph{section}.\arabic{proposition}}
\makeatletter
\@addtoreset{figure}{section} \@addtoreset{table}{section}
\makeatother

\input{appendix}


\end{document}

%% file: math_commands.tex
\usepackage{amsmath,amsfonts,bm}

\def\eqref#1{equation~\ref{#1}}

\def\1{\bm{1}}
\newcommand{\train}{\mathcal{D}}

\def\rvr{{\mathbf{r}}}

\def\rvx{{\mathbf{x}}}

\def\rvz{{\mathbf{z}}}

\def\ervx{{\textnormal{x}}}

\def\ervz{{\textnormal{z}}}

\def\vzero{{\bm{0}}}

\def\vmu{{\bm{\mu}}}

\def\vc{{\bm{c}}}

\def\ve{{\bm{e}}}

\def\vm{{\bm{m}}}

\def\vs{{\bm{s}}}

\def\vu{{\bm{u}}}

\def\mA{{\bm{A}}}

\def\mI{{\bm{I}}}

\def\mSigma{{\bm{\Sigma}}}

\DeclareMathAlphabet{\mathsfit}{\encodingdefault}{\sfdefault}{m}{sl}
\SetMathAlphabet{\mathsfit}{bold}{\encodingdefault}{\sfdefault}{bx}{n}

\newcommand{\E}{\mathbb{E}}
\newcommand{\Ls}{\mathcal{L}}
\newcommand{\R}{\mathbb{R}}

\newcommand{\Var}{\mathrm{Var}}

\newcommand{\Cov}{\mathrm{Cov}}


%% file: tables/tab_synthetic_gauss_d200.tex
\begin{tabular}{w{l}{74pt}@{\hspace{2em}}w{c}{58pt}w{c}{51pt}w{c}{21pt}@{\hspace{2em}}w{c}{58pt}w{c}{51pt}w{c}{21pt}}
\toprule
& \multicolumn{3}{c}{Gaussian} & \multicolumn{3}{c}{NSF} \\ \cmidrule(lr){2-4} \cmidrule(lr){5-7}
Method or reference & Test LL $\uparrow$ & $\chi^2_d$ $(\to 1)$ & MSE $\downarrow$ & Test LL $\uparrow$ & $\chi^2_d$ $(\to 1)$ & MSE $\downarrow$ \\
\midrule
Prior & $-283.79$ & $1.000$ & $1.000$ & $-283.79$ & $1.000$ & $1.000$ \\
Posterior & $-122.84$ & $1.000$ & $0.200$ & $-122.84$ & $1.000$ & $0.200$ \\
\midrule
Std-NPE & $-141.30 \pm 0.11$ & $1.095 \pm 0.001$ & $0.234$ & $-169.66 \pm 0.54$ & $0.979 \pm 0.004$ & $0.306$ \\
Std-NPE-mixture, $K'{=}10$ & $-134.79 \pm 0.12$ & $0.890 \pm 0.001$ & $0.212$ & $-159.24 \pm 0.27$ & $0.677 \pm 0.002$ & $0.251$ \\
Std-NPE-mixture, $K'{=}30$ & $-133.28 \pm 0.11$ & $0.875 \pm 0.001$ & $0.211$ & $-156.61 \pm 0.14$ & $0.658 \pm 0.001$ & $0.246$ \\
\midrule
CF-NPE, $M{=}1$ & $-133.92 \pm 0.13$ & $\mathbf{0.994 \pm 0.002}$ & $0.224$ & $-133.92 \pm 0.13$ & $\mathbf{1.000 \pm 0.001}$ & $0.224$ \\
CF-NPE, $M{=}3$ & $-128.42 \pm 0.10$ & $\mathbf{0.993 \pm 0.002}$ & $0.212$ & $-128.31 \pm 0.10$ & $\mathbf{1.000 \pm 0.001}$ & $0.212$ \\
CF-NPE, $M{=}10$ & $\mathbf{-126.41 \pm 0.10}$ & $\mathbf{0.993 \pm 0.002}$ & $\mathbf{0.208}$ & $\mathbf{-126.22 \pm 0.10}$ & $\mathbf{1.000 \pm 0.001}$ & $\mathbf{0.207}$ \\
CF-NPE-mixture, $M{=}3$ & $-128.41 \pm 0.10$ & $\mathbf{0.994 \pm 0.001}$ & $0.212$ & $-128.30 \pm 0.10$ & $\mathbf{1.001 \pm 0.001}$ & $0.212$ \\
\bottomrule
\end{tabular}

%% file: tables/tab_synthetic_funnel_d200.tex
\begin{tabular}{w{l}{74pt}@{\hspace{2em}}w{c}{58pt}w{c}{51pt}w{c}{21pt}@{\hspace{2em}}w{c}{58pt}w{c}{51pt}w{c}{21pt}}
\toprule
& \multicolumn{3}{c}{Gaussian} & \multicolumn{3}{c}{NSF} \\ \cmidrule(lr){2-4} \cmidrule(lr){5-7}
Method or reference & Test LL $\uparrow$ & $\chi^2_d$ $(\to 1)$ & MSE $\downarrow$ & Test LL $\uparrow$ & $\chi^2_d$ $(\to 1)$ & MSE $\downarrow$ \\
\midrule
Prior & $-283.79$ & $1.000$ & $1.000$ & $-283.79$ & $1.000$ & $1.000$ \\
Posterior & $178.68$ & $1.000$ & $0.017$ & $178.68$ & $1.000$ & $0.017$ \\
\midrule
Std-NPE & $-101.17 \pm 6.42$ & $1.505 \pm 0.075$ & $0.134$ & $-58.98 \pm 5.84$ & $\mathbf{0.995 \pm 0.011}$ & $0.242$ \\
Std-NPE-mixture, $K'{=}10$ & $-62.49 \pm 1.21$ & $0.285 \pm 0.004$ & $0.055$ & $-20.30 \pm 1.75$ & $0.231 \pm 0.002$ & $0.105$ \\
Std-NPE-mixture, $K'{=}30$ & $-56.01 \pm 1.28$ & $0.249 \pm 0.002$ & $\mathbf{0.049}$ & $-9.90 \pm 1.52$ & $0.198 \pm 0.002$ & $0.096$ \\
\midrule
CF-NPE, $M{=}1$ & $15.28 \pm 0.88$ & $\mathbf{1.007 \pm 0.005}$ & $0.076$ & $38.09 \pm 0.67$ & $1.068 \pm 0.005$ & $0.045$ \\
CF-NPE, $M{=}3$ & $56.36 \pm 1.00$ & $\mathbf{1.009 \pm 0.004}$ & $0.057$ & $83.82 \pm 0.73$ & $1.107 \pm 0.006$ & $0.031$ \\
CF-NPE, $M{=}10$ & $\mathbf{78.90 \pm 1.03}$ & $\mathbf{0.998 \pm 0.008}$ & $\mathbf{0.050}$ & $\mathbf{114.09 \pm 0.92}$ & $1.154 \pm 0.014$ & $\mathbf{0.026}$ \\
CF-NPE-mixture, $M{=}3$ & $58.08 \pm 0.97$ & $\mathbf{0.990 \pm 0.004}$ & $0.057$ & $89.64 \pm 0.70$ & $1.059 \pm 0.004$ & $0.030$ \\
\bottomrule
\end{tabular}

%% file: tables/tab_synthetic_mixture_d20.tex
\begin{tabular}{w{l}{74pt}@{\hspace{2em}}w{c}{58pt}w{c}{51pt}w{c}{21pt}@{\hspace{2em}}w{c}{58pt}w{c}{51pt}w{c}{21pt}}
\toprule
& \multicolumn{3}{c}{Gaussian} & \multicolumn{3}{c}{NSF} \\ \cmidrule(lr){2-4} \cmidrule(lr){5-7}
Method or reference & Test LL $\uparrow$ & $\chi^2_d$ $(\to 1)$ & MSE $\downarrow$ & Test LL $\uparrow$ & $\chi^2_d$ $(\to 1)$ & MSE $\downarrow$ \\
\midrule
Prior & $-28.38$ & $1.000$ & $1.000$ & $-28.38$ & $1.000$ & $1.000$ \\
Posterior & $0.07$ & $1.000$ & $0.087$ & $0.07$ & $1.000$ & $0.087$ \\
\midrule
Std-NPE & $-4.45 \pm 0.03$ & $\mathbf{1.007 \pm 0.003}$ & $0.095$ & $-3.02 \pm 0.03$ & $0.815 \pm 0.004$ & $0.104$ \\
Std-NPE-mixture, $K'{=}10$ & $-4.13 \pm 0.02$ & $0.956 \pm 0.003$ & $0.093$ & $-2.48 \pm 0.03$ & $0.721 \pm 0.002$ & $0.098$ \\
Std-NPE-mixture, $K'{=}30$ & $-4.10 \pm 0.02$ & $0.953 \pm 0.003$ & $0.093$ & $-2.41 \pm 0.03$ & $0.713 \pm 0.002$ & $0.097$ \\
\midrule
CF-NPE, $M{=}1$ & $-4.05 \pm 0.03$ & $\mathbf{0.994 \pm 0.004}$ & $0.092$ & $-0.81 \pm 0.02$ & $\mathbf{1.001 \pm 0.005}$ & $0.092$ \\
CF-NPE, $M{=}3$ & $-3.88 \pm 0.03$ & $\mathbf{0.995 \pm 0.004}$ & $0.090$ & $-0.63 \pm 0.03$ & $\mathbf{1.008 \pm 0.010}$ & $0.090$ \\
CF-NPE, $M{=}10$ & $\mathbf{-3.82 \pm 0.03}$ & $\mathbf{0.992 \pm 0.003}$ & $\mathbf{0.090}$ & $-0.55 \pm 0.02$ & $1.013 \pm 0.006$ & $\mathbf{0.090}$ \\
CF-NPE-mixture, $M{=}3$ & $-3.88 \pm 0.03$ & $0.987 \pm 0.004$ & $0.090$ & $\mathbf{-0.49 \pm 0.03}$ & $1.009 \pm 0.004$ & $\mathbf{0.090}$ \\
\bottomrule
\end{tabular}

%% file: tables/tab_strong_lensing.tex
\begin{tabular}{lcccc}
\toprule
Method & Test LL $\uparrow$ & $\chi^2_d$ $(\to 1)$ & MSE $\downarrow$ & Runtime (s) $\downarrow$ \\
\midrule
Std-NPE & $2.59\ [-0.67, 3.40]$ & $1.196 \pm 0.005$ & $0.124 \pm 0.001$ & 447 \\
Std-NPE-mixture, $K'{=}10$ & $13.49\ [13.14, 16.27]$ & $0.875 \pm 0.015$ & $0.111 \pm 0.001$ & 3906 \\
Std-NPE-mixture, $K'{=}50$ & $29.83\ [29.43, 32.03]$ & $0.810 \pm 0.022$ & $0.104 \pm 0.001$ & 20834 \\
\midrule
CF-NPE, $M{=}1$ & $60.05\ [58.82, 63.03]$ & $\mathbf{1.025 \pm 0.021}$ & $0.050 \pm 0.001$ & $3343{+}265$ \\
CF-NPE, $M{=}3$ & $83.21\ [80.27, 84.61]$ & $\mathbf{0.985 \pm 0.020}$ & $\mathbf{0.043 \pm 0.000}$ & $10009{+}265$ \\
CF-NPE-mixture, $M{=}3$ & $\mathbf{88.08\ [85.01, 90.26]}$ & $\mathbf{0.997 \pm 0.015}$ & $\mathbf{0.043 \pm 0.000}$ & $10009{+}2871$ \\
\bottomrule
\end{tabular}

%% file: appendix.tex
\makeatletter
\setlength{\@fptop}{0pt}
\setlength{\@fpbot}{0pt plus 1fil}
\makeatother
\newcommand{\appendixtablestyle}{%
  \small
  \setlength{\tabcolsep}{4pt}%
  \renewcommand{\arraystretch}{1.1}%
}

\section{Calibration of CF-NPE}
\label{app:oof}

\cfnpecalibration*

\begin{proof}
Condition on $\rvx_{1:N}$ and let $\vu_i:=\rvz_i-\vmu^\star(\rvx_i)$,
$\vu'_i:=\rvz'_i-\vmu^\star(\rvx_i)$, and $\ve_i:=\vc_i-\vmu^\star(\rvx_i)$, with
$\vmu^\star(\rvx):=\E[\rvz\mid \rvx]$. The squared residuals satisfy
\[
(\ervz_{i\ell}-c_{i\ell})^2 =u_{i\ell}^2-2u_{i\ell}e_{i\ell}+e_{i\ell}^2
\]
and
\[
(\ervz'_{i\ell}-c_{i\ell})^2 =(u'_{i\ell})^2-2u'_{i\ell}e_{i\ell}+e_{i\ell}^2.
\]
Because the second moments are finite, all the expectations below exist.
Given $\rvx_{1:N}$, $\vu'_i$ is independent of $\vc_i$ with mean zero, so
$\E[u'_{i\ell}e_{i\ell}\mid \rvx_{1:N}]=0$. Because the pairs are independent,
$\rvz_i\mid\rvx_{1:N}\sim p(\rvz\mid\rvx_i)$, so $u_{i\ell}$ and $u'_{i\ell}$ have
the same conditional law and hence equal conditional second moments.
Subtracting the two expansions and taking conditional expectations gives
\[
\E[(\ervz_{i\ell}-c_{i\ell})^2-(\ervz'_{i\ell}-c_{i\ell})^2\mid \rvx_{1:N}]
=-2\,\E[u_{i\ell}e_{i\ell}\mid \rvx_{1:N}]
=-2\Cov(\ervz_{i\ell},c_{i\ell}\mid \rvx_{1:N}),
\]
the last step because
$\E[u_{i\ell}\mid \rvx_{1:N}]=0$ and $\vmu^\star(\rvx_i)$ is fixed given $\rvx_{1:N}$.
Negating both sides gives \Cref{eq:standard_npe_cross_term}.
\end{proof}

\cfnperesidualmoment*

\begin{proof}
Fix $\rvx$ with $\E[\ervz_\ell^2\mid \rvx]<\infty$, which holds for almost every $\rvx$.
Adding and subtracting $\vmu^\star_\ell(\rvx)$ inside the square gives
\[
(\ervz_\ell-c_\ell)^2
=(\ervz_\ell-\vmu^\star_\ell(\rvx))^2
+2(\ervz_\ell-\vmu^\star_\ell(\rvx))(\vmu^\star_\ell(\rvx)-c_\ell)
+(\vmu^\star_\ell(\rvx)-c_\ell)^2.
\]
Taking the conditional expectation given $\rvx$, the first term becomes
$\mSigma^\star_{\ell\ell}(\rvx)$, the cross term vanishes because
$\E[\ervz_\ell-\vmu^\star_\ell(\rvx)\mid \rvx]=0$ and $\vc$ is fixed, and the last
term is unchanged. This is \Cref{eq:residual_moment_target}.
\end{proof}

\phantomsection\label{app:scale_calibration_proof}
\cfnpescale*

\begin{proof}
For $\alpha>0$, let $q_\alpha$ denote $q$ with every variance multiplied by
$\alpha^2$ and the mean unchanged. By assumption, $q_\alpha\in\mathcal Q$.
Write $\rvr(\rvx,\rvz):=\operatorname{diag}(\sigma_q(\rvx))^{-1}(\rvz-\vmu_q(\rvx))$.
Rescaling the variances gives
$\log q_\alpha(\rvz\mid \rvx)-\log q(\rvz\mid \rvx)
=-d\log\alpha+\tfrac12(1-\alpha^{-2})\|\rvr(\rvx,\rvz)\|^2$.
Taking expectations under $P_{\mathrm{train}}$ yields
\begin{equation}
\label{eq:scale_identity}
\E_{P_{\mathrm{train}}}[\log q_\alpha(\rvz\mid \rvx)]
-\E_{P_{\mathrm{train}}}[\log q(\rvz\mid \rvx)]
=-d\log\alpha+\tfrac d2\left(1-\alpha^{-2}\right)\chi^2_d(q;P_{\mathrm{train}}).
\end{equation}
Since $q$ maximizes the training log-likelihood over $\mathcal Q$, the derivative
of \Cref{eq:scale_identity} with respect to $\alpha$ vanishes at $\alpha=1$.
Thus $0=-d+d\,\chi^2_d(q;P_{\mathrm{train}})$, giving the result.
\end{proof}

\section{Synthetic Experiments}
\label{app:synthetic_experiments}

\subsection{Generative Models}
\label{app:synthetic_problems}

We study three problems: the Gaussian--Gaussian problem, the funnel problem, and the mixture problem.
Unless stated otherwise, every experiment uses $N = 10{,}000$ training simulations, a test set of $1{,}000$ simulations,
the Gaussian and NSF posterior approximations, 100 training epochs, the hyperparameters selected in
\Cref{app:synthetic_tuning}, and $K = 10$ folds for CF-NPE.
Each method is run with 10 seeds, and tables report the mean $\pm$ standard error over seeds.
Each later study states only the settings it changes.

\paragraph{Gaussian--Gaussian problem.}
This conjugate model is a standard benchmark in the simulation-based inference literature \citep[e.g.,][]{lueckmann2021benchmarking}, in which prior, likelihood, and posterior are all Gaussian.
It has $d = d_\rvx = 200$. The generative model is
\begin{align*}
\rvz &\sim \mathcal{N}(0, \mI_{200}), &
\rvx \mid \rvz &\sim \mathcal{N}(\rvz, 0.25\, \mI_{200}),
\end{align*}
giving posterior $\rvz \mid \rvx \sim \mathcal{N}(0.8\,\rvx, 0.2\, \mI_{200})$.

\paragraph{Funnel problem.}
The funnel problem has a non-Gaussian posterior whose conditional scale varies by orders of
magnitude.
It has $d = d_\rvx = 200$. The generative model is
\begin{align*}
\ervz_1 &\sim \mathcal{N}(0, 1), &
\ervz_j &\sim \mathcal{N}(0, 1) \quad (j \ge 2), \\
\ervx_1 \mid \rvz &\sim \mathcal{N}(\ervz_1, \sigma^2), &
\ervx_j \mid \rvz &\sim \mathcal{N}\!\big(\ervz_j\, e^{\ervz_1/2}, \sigma^2\big)
  \quad (j \ge 2),
\end{align*}
with $\sigma = 0.1$. Given $\ervz_1$, the coordinates $\ervz_j$ ($j \ge 2$) have independent Gaussian posteriors, so
the posterior reduces to a one-dimensional integral over $\ervz_1$, which we evaluate numerically on a grid. We report the
resulting Bayes-optimal test log-likelihood and posterior-mean MSE, averaged over $2 \times 10^4$ simulated pairs
(Monte Carlo standard errors below $0.01$), and the expected log-likelihood of the prior, $\E_{p(\rvz)}[\log p(\rvz)]$.

\paragraph{Mixture problem.}
The mixture problem has a two-component Gaussian observation model.
It has $d = 20$ and $d_\rvx = 2d = 40$. The generative model is
\begin{align*}
\rvz &\sim \mathcal{N}(0, \mI_{d}), &
\rvx \mid \rvz &\sim \tfrac{1}{2}\mathcal{N}(\mA\rvz +\vs, \sigma^2 \mI_{d_\rvx})
  + \tfrac{1}{2}\mathcal{N}(\mA\rvz -\vs, \sigma^2 \mI_{d_\rvx}),
\end{align*}
with entries $A_{ij}\stackrel{\mathrm{iid}}{\sim}\mathcal N(0,1/d)$ drawn once
and then fixed across methods and experimental seeds within each
dimension, mode separation $\|\vs\| = 3.0$, and $\sigma = 0.3$. Because the prior and each
likelihood component are Gaussian and linear in $\rvz$, the posterior is a two-component
Gaussian mixture available in closed form, so we report the Bayes-optimal test log-likelihood
as a reference. We also report the expected log-likelihood under the prior.
The posterior covariance is non-diagonal, which our diagonal Gaussian posterior approximation cannot represent.

\subsection{Hyperparameter Search Spaces and Selected Settings}
\label{app:synthetic_tuning}

\Cref{tab:synthetic_tuning_search} gives the shared hyperparameter search spaces.
\Cref{tab:gauss_d200_tuning_protocol,tab:funnel_d200_tuning_protocol,tab:mixture_d20_tuning_protocol} report the selected settings for the three problems.
Each problem has one CF-NPE stage-1 mean-network configuration shared by the
Gaussian and NSF families, and separately selected Std-NPE and CF-NPE stage-2
configurations for each family.

The fitted stage-1 networks are held fixed during stage-2 tuning.
Within each problem and approximating family, the CF-NPE configurations share the selected stage-1 and
stage-2 hyperparameters and differ only in fold count, replicate count, or whether stage~2 forms a
mixture (CF-NPE-mixture).

\newpage
\begin{table}[H]
\centering
\appendixtablestyle
\caption{Hyperparameter search spaces for the three problems.
Intervals marked log are sampled log-uniformly.}
\label{tab:synthetic_tuning_search}
\begin{tabular}{lll}
\toprule
hyperparameter & Std-NPE / CF-NPE stage~1 & CF-NPE stage~2 \\
\midrule
learning rate & $[10^{-4},10^{-2}]$ log & $[10^{-4},10^{-2}]$ log \\
weight decay & $[10^{-6},10^{-3}]$ log & $[10^{-6},10^{-3}]$ log \\
width & $\{64,128,256,512\}$ & $\{64,128,256\}$ \\
depth & $\{1,2,3\}$ & $\{1,2,3\}$ \\
dropout & $[0,0.5]$ & $[0,0.5]$ \\
batch size & $\{64,128,256\}$ & $\{64,128,256\}$ \\
coupling transforms (NSF) & $\{2,3,5\}$ & $\{2,3,5\}$ \\
spline bins (NSF) & $\{8,16\}$ & $\{8,16\}$ \\
tail bound (NSF) & $\{3,5,7\}$ & $\{3,5,7\}$ \\
\bottomrule
\end{tabular}
\end{table}

\begin{table}[p]
\centering
\appendixtablestyle
\caption{Selected hyperparameters for the Gaussian--Gaussian problem.}
\label{tab:gauss_d200_tuning_protocol}
\begin{tabular}{lrrrrr}
\toprule
& CF-NPE stage~1 & \multicolumn{2}{c}{Gaussian} & \multicolumn{2}{c}{NSF} \\
\cmidrule(lr){3-4}\cmidrule(lr){5-6}
hyperparameter & both families & Std-NPE & CF-NPE stage~2 & Std-NPE & CF-NPE stage~2 \\
\midrule
learning rate & $2.38\times10^{-3}$ & $7.53\times10^{-4}$ & $7.33\times10^{-4}$ & $2.86\times10^{-3}$ & $7.20\times10^{-3}$ \\
weight decay & $1.04\times10^{-5}$ & $3.44\times10^{-4}$ & $3.32\times10^{-5}$ & $2.50\times10^{-6}$ & $1.20\times10^{-4}$ \\
width & 256 & 256 & 128 & 256 & 128 \\
depth & 1 & 1 & 3 & 1 & 2 \\
dropout & $0.00100$ & $0.0151$ & $0.476$ & $0.0402$ & $0.418$ \\
batch size & 128 & 128 & 256 & 64 & 64 \\
coupling transforms & --- & --- & --- & 5 & 2 \\
spline bins & --- & --- & --- & 16 & 16 \\
tail bound & --- & --- & --- & 7 & 7 \\
\bottomrule
\end{tabular}

\vspace{24pt}
\caption{Selected hyperparameters for the funnel problem.}
\label{tab:funnel_d200_tuning_protocol}
\begin{tabular}{lrrrrr}
\toprule
& CF-NPE stage~1 & \multicolumn{2}{c}{Gaussian} & \multicolumn{2}{c}{NSF} \\
\cmidrule(lr){3-4}\cmidrule(lr){5-6}
hyperparameter & both families & Std-NPE & CF-NPE stage~2 & Std-NPE & CF-NPE stage~2 \\
\midrule
learning rate & $3.81\times10^{-3}$ & $4.70\times10^{-4}$ & $2.64\times10^{-4}$ & $2.02\times10^{-3}$ & $1.82\times10^{-3}$ \\
weight decay & $4.03\times10^{-5}$ & $8.96\times10^{-6}$ & $2.90\times10^{-5}$ & $5.80\times10^{-5}$ & $1.81\times10^{-5}$ \\
width & 512 & 512 & 256 & 256 & 128 \\
depth & 1 & 3 & 3 & 1 & 1 \\
dropout & $0.00440$ & $0.00423$ & $0.218$ & $0.00507$ & $0.477$ \\
batch size & 64 & 64 & 64 & 64 & 64 \\
coupling transforms & --- & --- & --- & 5 & 5 \\
spline bins & --- & --- & --- & 8 & 8 \\
tail bound & --- & --- & --- & 7 & 5 \\
\bottomrule
\end{tabular}

\vspace{24pt}
\caption{Selected hyperparameters for the mixture problem.}
\label{tab:mixture_d20_tuning_protocol}
\begin{tabular}{lrrrrr}
\toprule
& CF-NPE stage~1 & \multicolumn{2}{c}{Gaussian} & \multicolumn{2}{c}{NSF} \\
\cmidrule(lr){3-4}\cmidrule(lr){5-6}
hyperparameter & both families & Std-NPE & CF-NPE stage~2 & Std-NPE & CF-NPE stage~2 \\
\midrule
learning rate & $7.66\times10^{-3}$ & $9.88\times10^{-3}$ & $8.48\times10^{-4}$ & $9.05\times10^{-4}$ & $2.83\times10^{-3}$ \\
weight decay & $1.41\times10^{-4}$ & $1.15\times10^{-4}$ & $9.07\times10^{-4}$ & $1.55\times10^{-4}$ & $7.90\times10^{-4}$ \\
width & 256 & 64 & 64 & 64 & 256 \\
depth & 1 & 1 & 3 & 1 & 3 \\
dropout & $0.00115$ & $0.00950$ & $0.318$ & $0.0200$ & $0.450$ \\
batch size & 64 & 128 & 256 & 64 & 128 \\
coupling transforms & --- & --- & --- & 2 & 5 \\
spline bins & --- & --- & --- & 16 & 16 \\
tail bound & --- & --- & --- & 5 & 5 \\
\bottomrule
\end{tabular}
\end{table}

\subsection{Training Time}
\label{app:synthetic_runtime}

\begin{table}[H]
\centering
\appendixtablestyle
\caption{Training time in seconds for one seed, by
problem, method, and approximating family, timed after a GPU warm-up so that one-time
startup costs are excluded. CF-NPE rows show stage~1 $+$ stage~2 runtimes.
Timings were measured on one NVIDIA GeForce RTX~5090 GPU.}
\label{tab:synthetic_runtime}
\input{tables/tab_synthetic_runtime}
\end{table}

\subsection{Calibration Diagnostics}
\label{app:synthetic_calibration}

For $\chi^2_d$, we estimate the posterior mean and coordinate-wise variance from 500 posterior samples and average
the standardized squared errors over test pairs and coordinates. Simulation-based
calibration (SBC) \citep{talts2018validating} uses 256 posterior samples per test pair; the reported SBC-KS
is the Kolmogorov--Smirnov distance from uniformity of the normalized marginal ranks,
pooled over pairs and coordinates. TARP \citep{lemos2023sampling} evaluates joint coverage using 250 posterior
samples and 10 prior-drawn reference points per test pair; the reported TARP KS is the Kolmogorov--Smirnov distance from
uniformity of its coverage statistic, the largest gap between the empirical coverage
curve and the nominal-coverage diagonal. 
The classifier two-sample test (C2ST)~\citep{lopezpaz2017revisiting} simulates
10,000 new pairs $(\rvz,\rvx)$ and pairs each $\rvx$ with one draw from the posterior
approximation. A two-layer, width-256 classifier is trained for 30 epochs on half of
the pairs to tell simulated $\rvz$ from posterior draws and is scored on the other half. Its target accuracy is $0.5$. \Cref{tab:synthetic_calibration} reports the results.

\begin{table}[H]
\centering
\scriptsize
\setlength{\tabcolsep}{2.5pt}
\renewcommand{\arraystretch}{1.05}
\caption{Calibration diagnostics for the three problems by stage-1 replicate count.}
\label{tab:synthetic_calibration}
\input{tables/tab_synthetic_calibration}
\end{table}

\subsection{Cross-Fitting Ablations}
\label{app:gg_ablation}

On the Gaussian--Gaussian problem, we remove ingredients of CF-NPE ($M=3$) one at a time
(\Cref{tab:gg_ablation}).
The \emph{in-fold residuals} variant fits stage~2 to in-sample residuals: pairs in fold $k$ take their
mean from the networks of fold $k+1$ (modulo $K$), which are trained on them.
\emph{Single mean fit} replaces the folds and replicates with one stage-1 network, trained on
80\% of the training set and early-stopped on the other 20\%; stage~2 is fit to its
in-sample residuals on the same 80\%.

The in-fold residuals variant leaves the posterior mean unchanged but makes the posterior overconfident,
as \Cref{sec:theory} predicts. The single mean fit is also overconfident and
has higher posterior-mean MSE.

\begin{table}[H]
\centering
\appendixtablestyle
\caption{Cross-fitting ablations on the Gaussian--Gaussian problem, for each
approximating family. MSE standard errors are at most $0.001$. The CF-NPE, $M=3$ rows are the runs
reported in \Cref{tab:synthetic_gauss_d200}.}
\label{tab:gg_ablation}
\footnotesize
\input{tables/tab_gg_ablation}
\end{table}

\subsection{Fold Count \texorpdfstring{$K$}{K}}
\label{app:gg_kfold}

Increasing the number of folds $K$ grows the effective training set used for conditional mean estimation.
\Cref{tab:gg_kfold} reports that mean test log-likelihood rises with $K$, with diminishing gains as $K$ grows.
Even the smallest tested fold count achieves higher test log-likelihood than standard NPE.

\begin{table}[H]
\centering
\appendixtablestyle
\caption{Fold-count scaling for CF-NPE on the Gaussian--Gaussian problem
(Gaussian approximation, $M=3$): test log-likelihood vs.\ number of folds $K$.}
\label{tab:gg_kfold}
\input{tables/tab_gg_kfold}
\end{table}

\subsection{Oversized Stage-2 Flow Ablation}
\label{app:flow_size}

How sensitive is CF-NPE to the size of its stage-2 normalizing flow? On the
Gaussian--Gaussian problem with $M=3$, we replace the tuned
stage-2 NSF with three larger flows, keeping stage~1 fixed. \emph{+Width} increases the
width, the number of coupling transforms, and the number of spline bins, and lowers the
learning rate. \emph{+Depth} adds two hidden layers to +Width. \emph{$-$Regularization} removes
dropout and weight decay from +Depth. Each flow uses early stopping on validation
log-likelihood. \Cref{tab:flow_size} reports the varied hyperparameters
and results, using the diagnostics of \Cref{app:synthetic_calibration}.
These oversized flows lose a little test log-likelihood and have somewhat worse SBC, TARP, and C2ST diagnostics than
the tuned flow, but their $\chi^2_d$ stays close to one.

\begin{table}[H]
\centering
\appendixtablestyle
\caption{Stage-2 NSF configurations for CF-NPE ($M=3$) on the Gaussian--Gaussian
problem: the hyperparameters that differ across configurations (top), the
epoch selected by early stopping out of 100 (middle), and the calibration diagnostics
(bottom). The Tuned column's hyperparameters repeat
\Cref{tab:gauss_d200_tuning_protocol}, and its diagnostics repeat
\Cref{tab:synthetic_gauss_d200,tab:synthetic_calibration}.}
\label{tab:flow_size}
\input{tables/tab_flow_size}
\end{table}

\subsection{Training-Set-Size Scaling}
\label{app:gg_ntrain}

On the Gaussian--Gaussian problem, \Cref{tab:gg_ntrain} compares standard NPE and CF-NPE
across a variety of training-set sizes.
Standard NPE is tuned separately at each $N$, whereas CF-NPE reuses the hyperparameters selected at $N=10{,}000$.
CF-NPE has a higher mean test log-likelihood at every tested $N$.

\begin{table}[H]
\centering
\appendixtablestyle
\caption{Training-set-size scaling on the Gaussian--Gaussian problem
(Gaussian approximation): test log-likelihood vs.\ $N$ for standard NPE and CF-NPE ($M=3$).}
\label{tab:gg_ntrain}
\input{tables/tab_gg_ntrain}
\end{table}

\subsection{Dimension Sweeps}
\label{app:synthetic_dimsweep}

We vary the latent dimension $d$ of each problem, re-tuning standard NPE and CF-NPE stage~2 at each
dimension by the procedure of \Cref{app:synthetic_tuning}; CF-NPE stage~1 keeps the configuration
selected for the problem's main experiment
(\Crefrange{tab:gaussian_dimsweep}{tab:mixture_dimsweep}).

\begin{table}[H]
\centering
\appendixtablestyle
\caption{Dimension sweep for the Gaussian--Gaussian problem (Gaussian approximation):
test log-likelihood, the difference $\Delta$ in test log-likelihood between CF-NPE and
standard NPE (computed within each seed, then averaged), and the expected log-likelihoods of the prior
and the exact posterior.}
\label{tab:gaussian_dimsweep}
\input{tables/tab_gaussian_dimsweep}
\end{table}

\begin{table}[H]
\centering
\appendixtablestyle
\caption{Dimension sweep for the funnel problem (NSF approximation):
test log-likelihood, the difference $\Delta$ in test log-likelihood between CF-NPE and
standard NPE (computed within each seed, then averaged), and the expected log-likelihoods of the prior
and the exact posterior.}
\label{tab:funnel_dimsweep}
\input{tables/tab_funnel_dimsweep}
\end{table}

\begin{table}[H]
\centering
\appendixtablestyle
\caption{Dimension sweep for the mixture problem (NSF approximation):
test log-likelihood, the difference $\Delta$ in test log-likelihood between CF-NPE and
standard NPE (computed within each seed, then averaged), and the expected log-likelihoods of the prior
and the exact posterior.}
\label{tab:mixture_dimsweep}
\input{tables/tab_mixture_dimsweep}
\end{table}

\subsection{Temperature-Scaling Baseline}
\label{app:gg_temperature}

Adapting temperature scaling from classification~\citep{guo2017calibration}, we fit
a single scalar $T$ on $10{,}000$ newly simulated validation pairs after training and multiply every predicted
standard deviation by $T$.
On the Gaussian--Gaussian problem, temperature scaling brings
$\chi^2_d$ close to one but recovers only a small part of the test-log-likelihood
gap between standard NPE and CF-NPE (\Cref{tab:gg_temperature}). It leaves the
posterior-mean MSE unchanged by construction, so the MSE remains higher than CF-NPE's.

\begin{table}[H]
\centering
\appendixtablestyle
\caption{Temperature-scaling baseline on the Gaussian--Gaussian
problem: test log-likelihood, $\chi^2_d$, and posterior-mean MSE. MSE standard
errors are at most $0.001$. The Std-NPE and CF-NPE, $M=3$ rows are repeated from
\Cref{tab:synthetic_gauss_d200}.}
\label{tab:gg_temperature}
\input{tables/tab_gg_temperature}
\end{table}

\subsection{Alternative-Loss Baselines}
\label{app:gg_gradient}

Heteroscedastic-regression losses modify the Gaussian negative log-likelihood (NLL) so
that the predicted variance does not slow the fit of the mean, but they still fit the
variance to residuals of a mean trained on the same pairs. Does changing the loss alone
remove standard NPE's overconfidence? We train standard NPE with two such losses
(\Cref{tab:gg_gradient}). Both keep the Gaussian approximation and the hyperparameters
selected for standard NPE on the Gaussian--Gaussian problem.
\emph{$\beta$-NLL} with $\beta = 1$, following \citet{seitzer2022pitfalls}, multiplies each
coordinate's NLL by that coordinate's predicted variance, treated as a constant.
\emph{Stop-gradient}, adapted from \citet{stirn2023faithful}, adds an MSE loss for the mean
to a Gaussian NLL in which the mean is treated as a constant; the variance's gradients still
update the encoder that the mean and variance outputs share. $\beta$-NLL raises test
log-likelihood by about 0.9 nats, under a tenth of standard NPE's gap to CF-NPE, and neither loss
changes $\chi^2_d$ or posterior-mean MSE appreciably. In particular, both remain overconfident, like CF-NPE with
in-fold residuals (\Cref{app:gg_ablation}).

\begin{table}[H]
\centering
\appendixtablestyle
\caption{Alternative-loss baselines on the Gaussian--Gaussian problem:
test log-likelihood, $\chi^2_d$, and posterior-mean MSE. MSE standard errors
are at most $0.001$. The Std-NPE and CF-NPE, $M=3$ rows are
repeated from \Cref{tab:synthetic_gauss_d200}.}
\label{tab:gg_gradient}
\input{tables/tab_gg_gradient}
\end{table}

\subsection{Split-by-Fold Mixture Baseline}
\label{app:gg_split_by_fold}

Does validating each Std-NPE-mixture component on its own fold, as CF-NPE's stage-1
networks are validated, close the gap to CF-NPE? On the Gaussian--Gaussian problem,
\emph{split by fold} early-stops component $k$ of the $K'=10$ Std-NPE-mixture on fold $k$
and trains it on the other nine folds, instead of giving every component the same split
(\Cref{tab:gg_split_by_fold}). Splitting by fold raises the test log-likelihood of
Std-NPE-mixture, though not to the level of CF-NPE, and leaves its $\chi^2_d$ unchanged.

\begin{table}[H]
\centering
\appendixtablestyle
\caption{Split-by-fold mixture baseline on the Gaussian--Gaussian problem, for each
approximating family. MSE standard errors are at most $0.002$. The Std-NPE-mixture,
$K'=10$ and CF-NPE, $M=3$ rows are the runs reported in \Cref{tab:synthetic_gauss_d200}.}
\label{tab:gg_split_by_fold}
\footnotesize
\setlength{\tabcolsep}{2.5pt}
\input{tables/tab_gg_split_by_fold}
\end{table}

\section{Strong-Lensing Source Reconstruction: Implementation Details}
\label{app:exp:strong_lensing}

\paragraph{Forward simulator.}
The latent $\rvz$ is a synthetic $16{\times}16$ source-plane field on a $1''$
field of view ($d = 256$). The source field is drawn from a Gaussian
process on the source grid with a radial basis function kernel whose length scale is
one source pixel. The resulting pixels are spatially correlated, and each has a
unit-variance Gaussian marginal. The lens consists
of a singular isothermal ellipsoid~\citep{kormann1994isothermal} with
parameters $(\theta_E, e_1, e_2) \in [0.8, 1.2] \times [-0.25, 0.25]^2$
and an external shear $(\gamma_1, \gamma_2) \in [-0.1, 0.1]^2$, sampled
uniformly and independently per pair. Given the lens parameters, the image-plane value
at pixel $\theta$ is the source-field value at the deflected
position $\beta = \theta - \alpha(\theta)$, computed by bilinear
interpolation on the source grid using the closed-form SIE
deflection~\citep{birrer2018lenstronomy}. The image is rendered on a $64{\times}64$
grid with $0.05''$ pixels (a $3.2''$ field of view), convolved with a Gaussian
point-spread function with full width at half maximum $0.1''$, and corrupted with
Gaussian readout noise ($\sigma = 0.05$).

\paragraph{Networks.}
Every method receives only the lensed image; the lens parameters are not
provided. The image is transformed by $\operatorname{asinh}$ and processed by a residual
convolutional encoder with two stride-2 stages that reduce the spatial resolution
from $64{\times}64$ to the $16{\times}16$ source grid. Its base channels double
at each stage, and it uses group normalization and spatial dropout after each
residual stage. The CF-NPE stage-1
conditional-mean network maps this feature map to the predicted mean at the fixed
$16{\times}16$ resolution with $16$ residual blocks of width $256$, one
self-attention layer ($4$ heads), a sin/cos positional encoding, and extra
channels holding the image-wide average of each feature, ending in a layer that outputs
one mean per pixel.

Standard NPE and the CF-NPE stage-2 model both represent the source-grid density
with a coupling-based neural spline flow. Each coupling holds half of the source
pixels fixed under a spatial partition and transforms the rest with monotonic
rational-quadratic splines whose per-pixel parameters come from a conditioner
network: a residual network at the source-grid resolution that takes the fixed
half of the latent variables (placed at their grid positions, with zeros at the pixels being
transformed), the encoder
feature map, and a binary mask channel. Alternating couplings transform every
pixel.
The stage-2 model has its own encoder and models the out-of-fold residuals, each
divided by the standard deviation of its pixel's training residuals, conditional on
the image and the out-of-fold stage-1 mean. It passes the stage-1 mean through two
convolutions and concatenates the result with its own feature map on the source grid.
At inference, residual samples are multiplied back by those standard deviations and
added to the stage-1 mean.

\paragraph{Hyperparameter selection.}
Standard NPE is tuned by a $100$-trial tree-structured Parzen estimator study
selecting encoder and flow settings by held-out log-likelihood on a $10\%$
validation split at a fixed seed. CF-NPE stage~1 is tuned by held-out MSE
over the encoder, the layers after the encoder, and optimizer settings. CF-NPE stage~2 is tuned by a $100$-trial fixed-seed study on one fixed set of
trained $M=3$ stage-1 networks. It searches over learning rate, weight decay, and the
channel counts of the stage-2 encoder output and of the convolved stage-1 mean, selecting
by held-out log-likelihood of the standardized residuals. The flow architecture is held
at the standard-NPE selection and the encoder architecture at the stage-1 selection.
\Cref{tab:strong_lensing_tuning_protocol} lists the selected numerical settings.

\begin{table}[tbp]
\centering
\appendixtablestyle
\caption{Strong lensing ($d=256$): selected hyperparameters.
A dash denotes a setting that does not apply.}
\label{tab:strong_lensing_tuning_protocol}
\begin{tabular}{lrrr}
\toprule
hyperparameter & Std-NPE & CF-NPE stage~1 & CF-NPE stage~2 \\
\midrule
learning rate & $4.13\times10^{-4}$ & $1.38\times10^{-3}$ & $1.81\times10^{-3}$ \\
weight decay & $2.92\times10^{-4}$ & $1.00\times10^{-6}$ & $1.00\times10^{-5}$ \\
batch size & 128 & 64 & 128 \\
encoder base channels & 64 & 32 & 32 \\
encoder activation & GELU & SiLU & SiLU \\
encoder blocks per stage & 2 & 4 & 4 \\
encoder dropout & 0.1 & 0 & 0 \\
encoder stochastic depth & 0 & 0.2 & 0.2 \\
encoder group-norm groups & 32 & 8 & 8 \\
encoder squeeze-and-excitation & no & yes & yes \\
dropout after the encoder & --- & 0.3 & --- \\
coupling transforms & 3 & --- & 3 \\
spline bins & 8 & --- & 8 \\
tail bound & 5 & --- & 5 \\
coupling partition & column stripes & --- & column stripes \\
conditioner width & 64 & --- & 64 \\
conditioner blocks & 6 & --- & 6 \\
conditioner attention layers & 2 & --- & 2 \\
conditioner dropout & 0.1 & --- & 0.1 \\
conditioner stochastic depth & 0.1 & --- & 0.1 \\
stage-2 encoder channels & --- & --- & 64 \\
stage-1-mean channels & --- & --- & 8 \\
\bottomrule
\end{tabular}
\end{table}

\paragraph{Training and evaluation.}
Training uses $N = 10{,}000$ pairs for $100$ epochs with Adam. CF-NPE
uses $K = 10$ folds. All methods use 10 seeds and $1{,}000$ test pairs. Readout noise is resampled when examples are loaded. CF-NPE, $M=3$ and
CF-NPE-mixture, $M=3$ use the same trained stage-1 networks for each seed, so they
differ only in stage~2. Each member of the
$K'=50$ Std-NPE-mixture trains on its own random $90\%$ subsample of the training pairs. 
Runtimes were measured on a single NVIDIA GeForce RTX~5090 GPU.

\paragraph{Calibration.}
\Cref{tab:strong_lensing_calibration} reports the SBC and TARP diagnostics.

\begin{table}[tp]
\centering
\appendixtablestyle
\caption{Strong-lensing SBC and TARP diagnostics with an NSF approximation (mean $\pm$
standard error). C2ST is omitted because every method scores about $0.5$.
Log-likelihood, $\chi^2_d$, and
posterior-mean MSE for the same runs are in \Cref{tab:strong_lensing}.}
\label{tab:strong_lensing_calibration}
\input{tables/tab_strong_lensing_calibration}
\end{table}

%% file: tables/tab_synthetic_runtime.tex
\begin{tabular}{lrrrrrr}
\toprule
 & \multicolumn{2}{c}{Gaussian--Gaussian} & \multicolumn{2}{c}{Funnel} & \multicolumn{2}{c}{Mixture} \\ \cmidrule(lr){2-3} \cmidrule(lr){4-5} \cmidrule(lr){6-7}
Method & Gaussian & NSF & Gaussian & NSF & Gaussian & NSF \\
\midrule
Std-NPE & 2 & 83 & 9 & 131 & 8 & 22 \\
Std-NPE-mixture, $K'{=}10$ & 3 & 133 & 14 & 220 & 11 & 27 \\
Std-NPE-mixture, $K'{=}30$ & 3 & 344 & 27 & 565 & 12 & 66 \\
\midrule
CF-NPE, $M{=}1$ & $5{+}3$ & $5{+}75$ & $9{+}9$ & $9{+}135$ & $9{+}5$ & $9{+}72$ \\
CF-NPE, $M{=}3$ & $5{+}3$ & $5{+}68$ & $8{+}9$ & $8{+}134$ & $9{+}5$ & $9{+}72$ \\
CF-NPE, $M{=}10$ & $6{+}3$ & $6{+}67$ & $22{+}9$ & $22{+}136$ & $9{+}5$ & $9{+}71$ \\
CF-NPE-mixture, $M{=}3$ & $5{+}4$ & $5{+}143$ & $8{+}13$ & $8{+}196$ & $9{+}7$ & $9{+}93$ \\
\bottomrule
\end{tabular}

%% file: tables/tab_synthetic_calibration.tex
\begin{tabular}{l@{\hspace{0.8em}}ccc@{\hspace{0.8em}}ccc}
\toprule
& \multicolumn{3}{c}{Gaussian} & \multicolumn{3}{c}{NSF} \\ \cmidrule(lr){2-4} \cmidrule(lr){5-7}
Method & SBC-KS $\downarrow$ & TARP KS $\downarrow$ & C2ST ($\to 0.5$) & SBC-KS $\downarrow$ & TARP KS $\downarrow$ & C2ST ($\to 0.5$) \\
\midrule
\multicolumn{7}{l}{\textit{Gaussian--Gaussian problem}} \\
Std-NPE & $0.009 \pm 0.000$ & $0.239 \pm 0.003$ & $0.581 \pm 0.002$ & $0.005 \pm 0.001$ & $0.415 \pm 0.007$ & $0.827 \pm 0.024$ \\
\cmidrule(lr){1-7}
CF-NPE, $M{=}1$ & $\mathbf{0.002 \pm 0.000}$ & $\mathbf{0.017 \pm 0.002}$ & $0.557 \pm 0.002$ & $0.002 \pm 0.000$ & $0.022 \pm 0.002$ & $0.556 \pm 0.002$ \\
CF-NPE, $M{=}3$ & $\mathbf{0.002 \pm 0.000}$ & $0.106 \pm 0.004$ & $\mathbf{0.554 \pm 0.001}$ & $\mathbf{0.002 \pm 0.000}$ & $\mathbf{0.017 \pm 0.001}$ & $\mathbf{0.553 \pm 0.002}$ \\
\midrule
\multicolumn{7}{l}{\textit{Funnel problem}} \\
Std-NPE & $0.022 \pm 0.002$ & $\mathbf{0.311 \pm 0.006}$ & $0.872 \pm 0.004$ & $\mathbf{0.005 \pm 0.001}$ & $0.221 \pm 0.015$ & $0.973 \pm 0.002$ \\
\cmidrule(lr){1-7}
CF-NPE, $M{=}1$ & $\mathbf{0.005 \pm 0.001}$ & $0.347 \pm 0.004$ & $0.547 \pm 0.002$ & $0.012 \pm 0.001$ & $\mathbf{0.183 \pm 0.004}$ & $0.531 \pm 0.003$ \\
CF-NPE, $M{=}3$ & $\mathbf{0.005 \pm 0.001}$ & $0.450 \pm 0.004$ & $\mathbf{0.541 \pm 0.001}$ & $0.018 \pm 0.000$ & $0.203 \pm 0.005$ & $\mathbf{0.514 \pm 0.002}$ \\
\midrule
\multicolumn{7}{l}{\textit{Mixture problem}} \\
Std-NPE & $\mathbf{0.007 \pm 0.001}$ & $\mathbf{0.019 \pm 0.002}$ & $0.797 \pm 0.003$ & $0.023 \pm 0.001$ & $0.058 \pm 0.003$ & $0.688 \pm 0.003$ \\
\cmidrule(lr){1-7}
CF-NPE, $M{=}1$ & $\mathbf{0.007 \pm 0.001}$ & $0.023 \pm 0.002$ & $0.781 \pm 0.004$ & $\mathbf{0.007 \pm 0.001}$ & $\mathbf{0.024 \pm 0.004}$ & $\mathbf{0.530 \pm 0.002}$ \\
CF-NPE, $M{=}3$ & $\mathbf{0.008 \pm 0.001}$ & $\mathbf{0.017 \pm 0.002}$ & $\mathbf{0.764 \pm 0.005}$ & $0.012 \pm 0.002$ & $\mathbf{0.022 \pm 0.002}$ & $\mathbf{0.528 \pm 0.002}$ \\
\bottomrule
\end{tabular}

%% file: tables/tab_gg_ablation.tex
\begin{tabular}{lcccccc}
\toprule
Method & LL $\uparrow$ & $\chi^2_d$ ($\to 1$) & MSE $\downarrow$ & SBC-KS $\downarrow$ & TARP KS $\downarrow$ & C2ST ($\to 0.5$) \\
\midrule
\multicolumn{7}{l}{\textit{Gaussian}} \\
CF-NPE, $M{=}3$ & $-128.42 \pm 0.10$ & $0.993 \pm 0.002$ & $0.212$ & $0.002 \pm 0.000$ & $0.106 \pm 0.004$ & $0.554 \pm 0.001$ \\
in-fold residuals & $-129.45 \pm 0.12$ & $1.150 \pm 0.002$ & $0.212$ & $0.017 \pm 0.000$ & $0.240 \pm 0.004$ & $0.563 \pm 0.002$ \\
single mean fit & $-136.01 \pm 0.14$ & $1.094 \pm 0.009$ & $0.227$ & $0.011 \pm 0.001$ & $0.082 \pm 0.007$ & $0.566 \pm 0.001$ \\
\midrule
\multicolumn{7}{l}{\textit{NSF}} \\
CF-NPE, $M{=}3$ & $-128.31 \pm 0.10$ & $1.000 \pm 0.001$ & $0.212$ & $0.002 \pm 0.000$ & $0.017 \pm 0.001$ & $0.553 \pm 0.002$ \\
in-fold residuals & $-129.55 \pm 0.10$ & $1.164 \pm 0.001$ & $0.212$ & $0.019 \pm 0.000$ & $0.135 \pm 0.003$ & $0.560 \pm 0.001$ \\
single mean fit & $-136.91 \pm 0.11$ & $1.188 \pm 0.002$ & $0.227$ & $0.024 \pm 0.001$ & $0.114 \pm 0.007$ & $0.575 \pm 0.005$ \\
\bottomrule
\end{tabular}

%% file: tables/tab_gg_kfold.tex
\begin{tabular}{rr}
\toprule
$K$ & test log-likelihood \\
\midrule
2 & $-132.93 \pm 0.11$ \\
4 & $-129.50 \pm 0.10$ \\
8 & $-128.54 \pm 0.11$ \\
10 & $-128.42 \pm 0.10$ \\
16 & $-128.18 \pm 0.11$ \\
\bottomrule
\end{tabular}

%% file: tables/tab_flow_size.tex
\begin{tabular}{lcccc}
\toprule
 & Tuned & +Width & +Depth & $-$Regularization \\
\midrule
learning rate & $7.20\times10^{-3}$ & $1.00\times10^{-3}$ & $1.00\times10^{-3}$ & $1.00\times10^{-3}$ \\
weight decay & $1.20\times10^{-4}$ & $1.20\times10^{-4}$ & $1.20\times10^{-4}$ & 0 \\
width & 128 & 1024 & 1024 & 1024 \\
depth & 2 & 2 & 5 & 5 \\
dropout & $0.418$ & $0.418$ & $0.418$ & 0 \\
coupling transforms & 2 & 8 & 8 & 8 \\
spline bins & 16 & 32 & 32 & 32 \\
\midrule
best epoch & $87 \pm 4$ & $8 \pm 1$ & $7 \pm 1$ & $4 \pm 0$ \\
\midrule
LL $\uparrow$ & $-128.31 \pm 0.10$ & $-128.83 \pm 0.08$ & $-128.62 \pm 0.10$ & $-128.88 \pm 0.08$ \\
$\chi^2_d$ ($\to 1$) & $1.000 \pm 0.001$ & $0.994 \pm 0.005$ & $0.991 \pm 0.006$ & $0.994 \pm 0.002$ \\
SBC-KS $\downarrow$ & $0.002 \pm 0.000$ & $0.009 \pm 0.002$ & $0.007 \pm 0.001$ & $0.006 \pm 0.001$ \\
TARP KS $\downarrow$ & $0.017 \pm 0.001$ & $0.079 \pm 0.015$ & $0.039 \pm 0.007$ & $0.037 \pm 0.007$ \\
C2ST ($\to 0.5$) & $0.553 \pm 0.002$ & $0.588 \pm 0.005$ & $0.567 \pm 0.002$ & $0.578 \pm 0.005$ \\
\bottomrule
\end{tabular}

%% file: tables/tab_gg_ntrain.tex
\begin{tabular}{rrr}
\toprule
$N$ & Std-NPE & CF-NPE ($M{=}3$) \\
\midrule
$1{,}000$ & $-240.26 \pm 0.28$ & $-176.67 \pm 0.13$ \\
$3{,}000$ & $-181.89 \pm 0.16$ & $-142.58 \pm 0.10$ \\
$10{,}000$ & $-141.30 \pm 0.11$ & $-128.42 \pm 0.10$ \\
$30{,}000$ & $-129.37 \pm 0.10$ & $-125.00 \pm 0.07$ \\
$100{,}000$ & $-124.99 \pm 0.08$ & $-123.83 \pm 0.09$ \\
\bottomrule
\end{tabular}

%% file: tables/tab_gaussian_dimsweep.tex
\begin{tabular}{rccccc}
\toprule
$d$ & Std-NPE & CF-NPE & $\Delta$ (CF$-$std) & Prior & Posterior \\
\midrule
$2$ & $-1.24 \pm 0.01$ & $-1.24 \pm 0.01$ & $0.00 \pm 0.00$ & $-2.84$ & $-1.23$ \\
$10$ & $-6.23 \pm 0.01$ & $-6.22 \pm 0.02$ & $0.01 \pm 0.01$ & $-14.19$ & $-6.14$ \\
$50$ & $-31.77 \pm 0.07$ & $-31.98 \pm 0.06$ & $-0.21 \pm 0.03$ & $-70.95$ & $-30.71$ \\
$100$ & $-64.72 \pm 0.07$ & $-63.84 \pm 0.05$ & $0.88 \pm 0.04$ & $-141.89$ & $-61.42$ \\
$200$ & $-141.30 \pm 0.11$ & $-128.42 \pm 0.10$ & $12.88 \pm 0.07$ & $-283.79$ & $-122.84$ \\
\bottomrule
\end{tabular}

%% file: tables/tab_funnel_dimsweep.tex
\begin{tabular}{rccccc}
\toprule
$d$ & Std-NPE & CF-NPE & $\Delta$ (CF$-$std) & Prior & Posterior \\
\midrule
$2$ & $1.76 \pm 0.01$ & $1.76 \pm 0.01$ & $0.00 \pm 0.00$ & $-2.84$ & $1.78$ \\
$10$ & $6.55 \pm 0.03$ & $7.59 \pm 0.05$ & $1.04 \pm 0.03$ & $-14.19$ & $8.94$ \\
$20$ & $10.29 \pm 0.08$ & $13.96 \pm 0.08$ & $3.67 \pm 0.09$ & $-28.38$ & $17.87$ \\
$50$ & $25.79 \pm 0.50$ & $26.06 \pm 0.26$ & $0.27 \pm 0.51$ & $-70.95$ & $44.69$ \\
\bottomrule
\end{tabular}

%% file: tables/tab_mixture_dimsweep.tex
\begin{tabular}{rccccc}
\toprule
$d$ & Std-NPE & CF-NPE & $\Delta$ (CF$-$std) & Prior & Posterior \\
\midrule
$2$ & $0.52 \pm 0.01$ & $0.55 \pm 0.01$ & $0.03 \pm 0.00$ & $-2.84$ & $0.56$ \\
$10$ & $-1.19 \pm 0.02$ & $-0.36 \pm 0.03$ & $0.82 \pm 0.03$ & $-14.19$ & $-0.16$ \\
$20$ & $-2.99 \pm 0.03$ & $-0.63 \pm 0.03$ & $2.36 \pm 0.03$ & $-28.38$ & $0.07$ \\
$50$ & $-4.40 \pm 0.06$ & $-3.02 \pm 0.06$ & $1.38 \pm 0.09$ & $-70.95$ & $-0.14$ \\
\bottomrule
\end{tabular}

%% file: tables/tab_gg_temperature.tex
\begin{tabular}{lccc}
\toprule
Method & Test LL $\uparrow$ & $\chi^2_d$ $(\to 1)$ & MSE $\downarrow$ \\
\midrule
Std-NPE & $-141.30 \pm 0.11$ & $1.095 \pm 0.001$ & $0.234$ \\
Std-NPE + temperature scaling & $-140.88 \pm 0.11$ & $\mathbf{1.002 \pm 0.001}$ & $0.234$ \\
\midrule
CF-NPE, $M{=}3$ & $\mathbf{-128.42 \pm 0.10}$ & $0.993 \pm 0.002$ & $\mathbf{0.212}$ \\
\bottomrule
\end{tabular}

%% file: tables/tab_gg_gradient.tex
\begin{tabular}{lccc}
\toprule
Method & Test LL $\uparrow$ & $\chi^2_d$ $(\to 1)$ & MSE $\downarrow$ \\
\midrule
Std-NPE & $-141.30 \pm 0.11$ & $1.095 \pm 0.001$ & $0.234$ \\
Std-NPE + $\beta$-NLL & $-140.41 \pm 0.11$ & $1.082 \pm 0.001$ & $0.233$ \\
Std-NPE + stop-gradient & $-141.09 \pm 0.12$ & $1.090 \pm 0.001$ & $0.235$ \\
\midrule
CF-NPE, $M{=}3$ & $\mathbf{-128.42 \pm 0.10}$ & $\mathbf{0.993 \pm 0.002}$ & $\mathbf{0.212}$ \\
\bottomrule
\end{tabular}

%% file: tables/tab_gg_split_by_fold.tex
\begin{tabular}{lcccccc}
\toprule
Method & LL $\uparrow$ & $\chi^2_d$ ($\to 1$) & MSE $\downarrow$ & SBC-KS $\downarrow$ & TARP KS $\downarrow$ & C2ST ($\to 0.5$) \\
\midrule
\multicolumn{7}{l}{\textit{Gaussian}} \\
Std-NPE-mixture, $K'{=}10$ & $-134.79 \pm 0.12$ & $0.890 \pm 0.001$ & $0.212$ & $0.016 \pm 0.000$ & $0.238 \pm 0.002$ & $0.571 \pm 0.001$ \\
split by fold & $-133.34 \pm 0.11$ & $0.879 \pm 0.001$ & $0.211$ & $0.017 \pm 0.000$ & $0.233 \pm 0.003$ & $0.560 \pm 0.001$ \\
\cmidrule(lr){1-7}
CF-NPE, $M{=}3$ & $-128.42 \pm 0.10$ & $0.993 \pm 0.002$ & $0.212$ & $0.002 \pm 0.000$ & $0.106 \pm 0.004$ & $0.554 \pm 0.001$ \\
\midrule
\multicolumn{7}{l}{\textit{NSF}} \\
Std-NPE-mixture, $K'{=}10$ & $-159.24 \pm 0.27$ & $0.677 \pm 0.002$ & $0.251$ & $0.041 \pm 0.000$ & $0.417 \pm 0.003$ & $0.735 \pm 0.005$ \\
split by fold & $-154.88 \pm 0.29$ & $0.684 \pm 0.001$ & $0.246$ & $0.040 \pm 0.000$ & $0.436 \pm 0.004$ & $0.721 \pm 0.002$ \\
\cmidrule(lr){1-7}
CF-NPE, $M{=}3$ & $-128.31 \pm 0.10$ & $1.000 \pm 0.001$ & $0.212$ & $0.002 \pm 0.000$ & $0.017 \pm 0.001$ & $0.553 \pm 0.002$ \\
\bottomrule
\end{tabular}

%% file: tables/tab_strong_lensing_calibration.tex
\begin{tabular}{lcc}
\toprule
Method & SBC-KS $\downarrow$ & TARP KS $\downarrow$ \\
\midrule
Std-NPE & $0.0224 \pm 0.0006$ & $\mathbf{0.027 \pm 0.004}$ \\
Std-NPE-mixture, $K'{=}10$ & $0.0182 \pm 0.0014$ & $0.040 \pm 0.002$ \\
Std-NPE-mixture, $K'{=}50$ & $0.0233 \pm 0.0008$ & $0.054 \pm 0.002$ \\
\midrule
CF-NPE, $M{=}1$ & $0.0090 \pm 0.0009$ & $0.066 \pm 0.010$ \\
CF-NPE, $M{=}3$ & $0.0107 \pm 0.0016$ & $0.060 \pm 0.010$ \\
CF-NPE-mixture, $M{=}3$ & $\mathbf{0.0064 \pm 0.0007}$ & $0.046 \pm 0.007$ \\
\bottomrule
\end{tabular}